\documentclass{article}

\usepackage{microtype}
\usepackage{graphicx}
\usepackage{subfigure}
\usepackage{booktabs} 

\usepackage{hyperref}

\usepackage[accepted]{icml2024}

\usepackage{amsmath}
\usepackage{amssymb}
\usepackage{mathtools}
\usepackage{amsthm}

\usepackage[capitalize,noabbrev]{cleveref}

\theoremstyle{plain}
\newtheorem{theorem}{Theorem}
\newtheorem{proposition}[theorem]{Proposition}
\newtheorem{remark}[theorem]{Remark}
\newtheorem{lemma}[theorem]{Lemma}
\newtheorem{corollary}[theorem]{Corollary}
\theoremstyle{definition}

\newtheorem{assumption}
{Assumption}

\usepackage[textsize=tiny]{todonotes}
\usepackage[utf8]{inputenc}

\usepackage{graphicx}
\usepackage{natbib}
\usepackage{physics}
\usepackage{enumitem}   
\usepackage{wrapfig}

\usepackage{cleveref}
\usepackage{float}
\usepackage{url}
\usepackage{nicefrac}

\definecolor{blush}{rgb}{0.87, 0.36, 0.51}
\usepackage{xargs}

\newcommand{\R}{\mathbb R}
\newcommand{\X}{\R^d}
\newcommand{\E}{\mathbb E}
\newcommand{\C}{\mathcal C}
\newcommand{\ke}{k_{\epsilon}}

\newcommand{\cL}{\mathcal L}

\newcommand{\tg}{\mu^{\star}}
\newcommand{\ps}[1]{\langle #1 \rangle}
\newcommand{\cF}{\mathcal{F}}
\newcommand{\cG}{\mathcal{G}}

\newcommand{\cU}{\mathcal{U}}
\newcommand{\cUe}{\cU_{\epsilon}}
\newcommand{\cGe}{\cG_{V_{\epsilon}}}
\newcommand{\cFe}{\cF_{\epsilon}}
\newcommand{\tVe}{V_{\epsilon}}
\newcommand{\hess}{\mathrm{H}}
\DeclareMathOperator{\Hess}{Hess}
\newcommand{\Id}{\mathop{\mathrm{Id}}\nolimits}
\newcommand{\ID}{\mathop{\mathrm{I}_d}\nolimits}
\newcommand{\J}{\mathrm{J}}

\newcommand{\cP}{\mathcal{P}}
\newcommand{\KL}{\mathop{\mathrm{KL}}\nolimits}
\newcommand{\TV}{\mathop{\mathrm{TV}}\nolimits}
\newcommand{\BW}{\mathop{\mathrm{BW}}\nolimits}

\newcommand{\h}{h_{\mu}^{\epsilon}}

\DeclareMathOperator*{\argmin}{argmin}

\newcommand{\g}{h}

\icmltitlerunning{Theoretical Guarantees for Variational Inference with Fixed-Variance  Mixture of Gaussians}

\begin{document}

\twocolumn[
\icmltitle{Theoretical Guarantees for Variational Inference with Fixed-Variance  Mixture of Gaussians}



\icmlsetsymbol{equal}{*}

\begin{icmlauthorlist}
\icmlauthor{Tom Huix}{cmap}
\icmlauthor{Anna Korba}{yyy}
\icmlauthor{Alain Durmus}{cmap}
\icmlauthor{Eric Moulines}{cmap}
\end{icmlauthorlist}

\icmlaffiliation{yyy}{ENSAE, CREST, IP Paris}
\icmlaffiliation{cmap}{CMAP, Ecole polytechnique}
\icmlcorrespondingauthor{Tom Huix}{tom.huix@polytechnique.edu}
\icmlcorrespondingauthor{Anna Korba}{anna.korba@ensae.fr}

\icmlkeywords{Variational inference, Mixture of Gaussians, Bayesian Machine Learning}

\vskip 0.3in
]



\printAffiliationsAndNotice{}  

\begin{abstract}
Variational inference (VI) is a popular approach in Bayesian inference, that looks for the best approximation of the posterior distribution within a parametric family, minimizing a loss that is typically the (reverse)  Kullback-Leibler (KL) divergence. 
Despite its empirical success, the theoretical properties of VI have only received attention recently, and mostly when the parametric family is the one of Gaussians. This work aims to contribute to the theoretical study of VI in the non-Gaussian case by investigating the setting of Mixture of Gaussians with fixed covariance and constant weights. In this view, VI over this specific family can be casted as the minimization of a Mollified relative entropy, i.e. the KL between the convolution (with respect to a Gaussian kernel) of an atomic measure supported on Diracs, and the target distribution. The support of the atomic measure corresponds to the localization of the Gaussian components. Hence, solving variational inference becomes equivalent to optimizing the positions of the Diracs (the particles), which can be done through gradient descent and takes the form of an interacting particle system. We study two sources of error of variational inference in this context when optimizing the mollified relative entropy. The first one is an optimization result, that is a descent lemma establishing that the algorithm decreases the objective at each iteration. The second one is an approximation error, that upper bounds the objective between an  optimal finite mixture and the target distribution.
\end{abstract}

\section{Introduction}
A fundamental problem in computational statistics and machine learning is to compute integrals with respect to some target probability distribution $\tg$ on $\R^d$ whose density is known only up to a normalization constant. For instance in Bayesian inference, $\tg$ is the posterior distribution over the parameters of complex models. 
The general goal of sampling methods is thus to provide an approximate distribution for which the integrals are easily computed. 
A large number of methods have been developed to tackle this problem. 
The classical approach is to sample the posterior using Markov Chain Monte Carlo (MCMC) algorithms, in which a Markov chain designed to converge to $\tg$ is simulated for a sufficiently long time \cite{roberts2004general}. These methods use the discrete measure over past iterates of the algorithm as an approximation of the posterior to compute integrals of interest. However, MCMC algorithms are generally computationally expensive, and it is an open problem to diagnose their convergence in practice \cite{moins2023use}.  
Variational inference (VI) \citep{blei2017variational} has emerged as a powerful and versatile alternative in Bayesian inference.  
By framing the problem as an optimization task, VI aims to find an approximate candidate distribution within a parametric family of distributions $\C$ that minimizes the (reverse) Kullback-Leibler (KL) divergence to the target:
\begin{equation}\label{eq:vi}
\hat{\nu}:=\argmin_{\mu \in \C}  \KL(\mu|\tg),
\end{equation}
where $\KL(\mu|\tg)=\int \log(\nicefrac{d\mu}{d\tg})d\mu$ if $\mu$ is absolutely continuous with respect to $\tg$ denoting $\nicefrac{d\mu}{d\tg}$ its Radon-Nikodym density, and $+\infty$ else; and $\hat{\nu}$ is referred to as the optimal approximation within the variational family.

While VI  methods can only return an approximation of the target, they are much more tractable in the large scale setting, since they benefit from efficient optimization methods, e.g. parallelization or stochastic optimization  \cite{zhang2018advances}. Hence, VI has proven effective in numerous applications and is a popular paradigm especially in high-dimensional scenarios. Still, the understanding of its theoretical properties remains a challenging and active area of research. 
Fundamentally, there are two sources of errors in VI: the \emph{approximation} error that quantifies how far $\hat{\nu}$ is from $\tg$, and the \emph{optimization} error that comes from the optimization of the objective in \eqref{eq:vi} to approach $\hat{\nu}$.

Even among the recent literature on theoretical guarantees for VI, most efforts have been concentrated in the case where $\C$ is the set of non-degenerate Gaussian distributions.  Recently, \citet{katsevich2023approximation} studied the approximation quality (in total variation) of the approximate posterior $\hat{\nu}$, i.e., minimizers of the objective \eqref{eq:vi}, and show that it better estimates the true mean and covariance of the posterior than the well-known Laplace approximation \citep{helin2022non}.  Regarding the optimization of \eqref{eq:vi}, still restricted to Gaussians, several recent works leverage the geometry of Wasserstein gradient flows, more precisely the equivalence between Bures-Wasserstein gradient flows on the space of probability distributions and Euclidean flows on the space of parameters of the variational approximation. They derive novel algorithms  with convergence guarantees e.g. through gradient-descent \citep{lambert2022variational} or forward-backward \citep{diao2023forward, domke2023provable} time discretizations; and precise connections with Black-Box Variational Inference  (BBVI) \citep{yi2023bridging}.  

However, to the best of our knowledge, the study of approximation and computational guarantees when $\C$ is a set of mixture of Gaussians has not been tackled yet. 
Mixture models are a widely used class of probabilistic models that capture complex and multi-modal data distributions by combining simpler components. Moreover, they are dense in the space of probability distributions with $p$ bounded moments in the Wasserstein-$p$ metric \citep[Lemma 3.1]{delon2020wasserstein}.

In this study, we propose to consider a simplified setting where the Gaussian components have equal weights and share the same diagonal covariance.  This regime breaks down the complexity of the problem, and is still theoretically challenging, but remains a practically relevant scenario. In this setting, variational inference aims to optimize the locations of the means of the Gaussian mixture to approximate the target distribution. 

\textbf{Contributions.} In this paper, we derive theoretical guarantees for variational inference for some mixture of Gaussians family. We leverage the framework of Wasserstein gradient flows as well as the smoothness of the optimization objective to derive a descent lemma, showing that the objective decreases at each discrete time iteration. Regarding the approximation quality of Gaussian mixtures in (reverse) KL divergence, we use a similar technique than \cite{li1999mixture} that established exact rates for the 
(forward) KL divergence, and we obtain upper bounds on the approximation error of VI in that setting.

This paper is organized as follows.  \Cref{sec:background} provides the relevant background on optimization over the space of probability distributions and introduces the mollified relative entropy that is the objective functional minimized in our context.  In \Cref{sec:optim_guarantees} we derive a descent lemma, establishing that the Wasserstein gradient descent algorithm decreases the objective at each iteration. In \Cref{sec:approx_guarantees} we focus on the approximation error that quantifies how well minimizers of the VI objective approach the target distribution, for a given number of mixture components.
In \Cref{sec:related_work} we connect our results with relevant works
in the Variational Inference literature.

\textbf{Notations.} %
We denote by 
$\cP_2(\R^d)$ the set of probability distributions on $\R^d$ with bounded second moments. 
Given a Lebesgue measurable map $ T: X\to X$ and $\mu\in \cP_2(X)$, $ T_{\#}\mu$ is the pushforward measure of $\mu$ by $T$.
 For any $\mu \in \cP_2(\X)$, $L^2(\mu)$ is the space of functions $f : \X \to \R^d$ such that $\int \|f\|^2 d\mu < \infty$. 
 We denote by $\Vert \cdot \Vert_{L^2(\mu)}$ and $\ps{\cdot,\cdot}_{L^2(\mu)}$ respectively the norm and the inner product of the Hilbert space $L^2(\mu)$. 
We consider, for $\mu,\nu \in \cP_2(\X)$, the 2-Wasserstein distance $W_2 (\mu, \nu) = \inf_{s \in \mathcal{S}(\mu,\nu)} \int \|x-y\|^2 ds(x,y)$, where $\mathcal{S}(\mu,\nu)$ is the set of couplings between $\mu$ and $\nu$. The metric space
$(\cP_2(\X), W_2)$ is called the Wasserstein space. 
We use $C^k(\R^d)$ to denote continuously $k$-differentiable functions and $C^\infty(\R^d)$ to indicate the smooth functions. The space of continuous $k$-differentiable functions with compact support on $X$ is $C_c^{k}(\R^d)$. 
If $\psi : \R^d \to \R^p$ is differentiable, we denote by $\J \psi : \R^d \to \R^{p \times d}$ its Jacobian.  If $p = 1$, we denote by $\nabla \psi$ the gradient of $\psi$.
Moreover, if $\nabla \psi$ is differentiable, the Jacobian of $\nabla \psi$ is the Hessian of $\psi$ denoted by $\hess \psi$. If $p = d$, $\div \psi$ denotes the divergence of $\psi$. We also denote by $\Delta \psi$ the Laplacian of $\psi$, where $\Delta \psi=\div\nabla \psi$. The Hilbert-Schmidt norm  is denoted $\|\cdot\|_{HS}$.

In the following, we assume that $\tg$ admits a density proportional to $\exp(-V)$ with respect to the Lebesgue measure over $\R^d$. 

\section{The mollified relative entropy}\label{sec:background}

Writing $\tg = e^{-V}/Z$ with $Z$ the unknown normalization constant, the (reverse) Kullback-Leibler divergence (or relative entropy) can be written as
\begin{align*}
    \KL(\mu|\tg)&=\int Vd\mu + \int \log(\mu)d\mu +\log(Z) \\
    &:= \cG_V(\mu)+\cU(\mu) +\log(Z),
\end{align*}
for $\mu$ absolutely continuous with respect to $\tg$, and $+\infty$ else. Hence, it
decomposes as the sum of a potential energy $\cG_V$, i.e. a linear functional, and the negative entropy $\cU$, up to an additive constant that is fixed in the optimization problem.

We now consider the minimization problem of Variational Inference \eqref{eq:vi} for mixture of Gaussians. We will study a specific setting where the variational family is the set of mixture of $n$ Gaussians with equally weighted components, and where these components have the same diagonal covariance $\epsilon^2 \ID$, for some $n\in \mathbb{N}^*,\epsilon>0$. 
\begin{equation*}
\C_n =\left\{ \frac{1}{n} \sum_{i=1}^n  q_i, \; q_i = \mathcal{N}(x_i,\epsilon^2 \ID),\; x_i \in \R^d\ \right\},
\end{equation*}
where $\ID$ denotes the $d$-dimensional identity matrix. 
In our setting, only the positions (the means) of the mixture components will be optimized. Hence, searching for the optimal distribution in the variational family approximating the target $\mu^*$
consists in finding the optimal locations  of the Gaussian components in $\R^d$. We will denote $\ke$ the normalized Gaussian kernel, i.e. $\ke(x) =\exp(-\|x\|^2/(2\epsilon^2)) Z_{\epsilon}^{-1}$, where $\int \ke(x)dx = 1$ and $Z_{\epsilon}\propto (\epsilon^2)^{d/2}$. It is a specific example of mollifiers, i.e. smooth  approximations
of the Dirac delta at the origin, as introduced in \cite{friedrichs1944identity}. 
For $\mu$ a given probability distribution on $\R^d$, we denote by $\ke \star \mu$ its convolution with the Gaussian kernel that writes $\ke\star \mu = \int \ke(\cdot - x)d\mu(x).$ 
Equipped with these notations,
we can write $\C_n = \left\{ \ke \star \mu_n, \; \mu_n =\frac{1}{n}\sum_{i=1}^n \delta_{x^i}, \; x^1,\dots, x^n \in \R^d \right\}$.

Irrespective of the number of components $n$, VI with Gaussian mixtures whose components share the same variance can be written more generally as minimizing \eqref{eq:vi} restricted to the family $\C=\left\{ \ke \star \mu,\; \mu \in \cP(\R^d) \right\}$. The latter problem can be then reformulated as the optimization over $\cP(\R^d)$ of the following objective functional, 
 that we will refer to as the \textit{mollified relative entropy} (or mollified KL):
\begin{align}\label{eq:mollified_e}
\cFe(\mu)&=\int V d (\ke\star \mu)+ \int \log(\ke\star \mu)d(\ke\star\mu)\nonumber\\
&:= \cGe(\mu) + \cUe(\mu),
\end{align}
where $\cGe$ is a potential energy with respect to a convoluted potential $\tVe = \ke \star V$ (using the associativity of the convolution operation), and $\cUe(\mu)=\cU(\ke \star \mu)$ is a functional that we will refer to as the mollified negative entropy. In contrast with the negative entropy defined above, the mollified one is well-defined for discrete measures.

\subsection{Algorithm}\label{sec:algorithm}

We now discuss the optimization of the mollified relative entropy, starting from the continuous time dynamics to the practical discrete-time particle scheme. 

A Wasserstein gradient flow of $\cFe$ \cite{ambrosio2008gradient} can be described by the following continuity equation:
\begin{equation}\label{eq:wgf}
    \frac{\partial \mu_t}{\partial t} = \div(\mu_t \nabla_{W_2}\cFe(\mu_t)),\;\nabla_{W_2}\cFe(\mu_t):=\nabla \cFe'(\mu_t),
\end{equation}
where $\cFe'$ denotes the first variation 
of $\cFe$.  Recall that if it exists, the first variation of a functional $\cF$ at $\nu$ is the function $\cF'(\nu):\R^d \rightarrow \R$ s. t. for $\nu, \mu \in \mathcal{P}(\R^d)$:
$\lim_{\epsilon \rightarrow 0}\nicefrac{1}{\epsilon}[\cF(\nu+\epsilon  (\mu-\nu)) -\cF(\nu)]=\int\cF'(\nu)(x) (d \mu(x)-d \nu(x))$. Wasserstein gradient flows are paths of steepest descent with respect to the $W_2$ metric, and can be seen as analog to Euclidean gradient flows on the space of probability distributions \citep{santambrogio2017euclidean}.

Starting from some initial distribution $\mu_0\in \cP(\R^d)$, and for some given step-size $\gamma>0$, a forward (or explicit) time-discretization of \eqref{eq:wgf} corresponds to the Wasserstein gradient descent algorithm, and can be written at each discrete time iteration $l\in \mathbb{N}$ as: 
\begin{equation}\label{eq:wgd}
\mu_{l+1}=(\Id-\gamma \nabla \cF_{\epsilon}'(\mu_l) )_{\#}\mu_{l}
\end{equation}
where $\Id$ is the identity map in $L^2(\mu_l)$.

For discrete measures $\mu_n=\nicefrac{1}{n}\sum_{i=1}^n \delta_{x^i}$, we can define the finite-dimensional objective $F(X^n):=\cFe(\mu_n)$ where $X^n=(x^1,\dots,x^n)$, since the functional $\cFe$ is well defined for discrete measures. The Wasserstein gradient descent dynamics of $\cFe$~\eqref{eq:wgd} then correspond to standard  gradient descent of the (finite-dimensional) function $F$, i.e., gradient descent on the position of the particles. In that setting, we recall that particles correspond to the means of the Gaussian components of the mixture. The gradient of $F$ is readily obtained as 
\begin{multline}
    \nabla_{x^j}F(X^n)= \int_{\R^d} \nabla V(y)\ke(y-x^{j})dy \\
    +\int_{\R^d} \frac{\sum_{i=1}^n \nabla \ke(y-x^{i})}{\sum_{i=1}^n  \ke(y-x^{i})}\ke(y-x^{j})dy.\label{eq:algoupdate}
\end{multline}
 Notice that the gradient above involves integrals over $\R^d$. However, using a Gaussian kernel $\ke$, since $\nabla \ke(x) = - \frac{x}{\epsilon^2}\ke(x)$, these integrals can be easily approximated through Monte Carlo using Gaussian samples. 
A particle version of  \eqref{eq:wgd}, e.g.,  starting with  $\mu_0$ discrete, can then be written as the following gradient descent iterates:
\begin{equation}\label{eq:particle_gd}
x_{l+1}^{j}= x_{l}^{j} - \gamma \nabla_{x_l^j}F(X^n_l)
\end{equation}
for $j=1,\dots,n$ and  where $X^n_l=(x^1_l,\dots,x^n_l)$. Hence, minimizing $\cFe$ on discrete measures results in a a particle system that interact through the gradient of the objective. The reader may refer to \Cref{sec:particle_implementation} for the detailed computations leading to the particle scheme. Notice that it recovers the scheme mentioned in \citep[Section 5]{lambert2022variational} where the covariance of the mixture components are fixed, see \Cref{sec:MOG_lambert} for a detailed discussion.


\begin{remark}
Notice that the Wasserstein gradient at $\mu\in \cP_2(\R^d)$ of the mollified KL in \Cref{eq:wgf}, $ \nabla \cFe'(\mu_t):\R^d\to \R^d$ writes for any $w\in \R^d$:
\begin{equation}\label{eq:wgd_cFe}
    \nabla \cFe'(\mu_t)(w) = \ke \star \nabla V(w) + \ke \star \nabla \log(\ke \star \mu) (w),
\end{equation}
see \Cref{sec:particle_implementation}.
Hence, it differs from the Wasserstein gradient of the (standard) KL w.r.t. $\tg\propto e^{-V}$, i.e. $\KL(\cdot|\tg)$ evaluated at the convoluted distribution that writes as $\nabla \log\left(\nicefrac{\ke \star \mu}{\tg}\right)$, see \citep[Section 3.1.3]{wibisono2018sampling}. 
\end{remark}


\subsection{Non-smoothness of the KL}\label{sec:kl_non_smooth}

In Euclidean optimization, it is standard that the convergence of gradient descent is guaranteed when the objective function is convex and smooth, which relates to a lower bound and upper bound on the Hessian of the objective when the latter is twice differentiable \cite{garrigos2023handbook}. Analogously, when optimizing a functional on the Wasserstein space,  lower and upper bounds on the Hessian characterize respectively convexity and smoothness on the functional $\cF$ with respect to the Wasserstein-2 geometry (see \citet[Proposition 16.2]{villani2009optimal}). The 
Wasserstein space has a Riemannian geometry \cite{otto2001geometry}, where one can define for any $\mu$ the tangent space $\mathcal{T}_{\mu}\cP_2(\R^d) =\overline{\{\nabla \psi,\enspace \psi\in C_c^{\infty}(\X)\}} \subset L^2(\mu)$ \citep[Definition 8.4.1]{ambrosio2008gradient}. The $W_2$ Hessian of a functional $\cF$, denoted $H\cF_{|\mu}$ is an operator over $\mathcal{T}_{\mu}\mathcal{P}_2(\X)$ verifying $\ps{H\cF_{|\mu}v_t, v_t}_{L^2(\mu)}=\frac{d^2}{dt^2}\Bigr|_{\substack{t=0}}\mathcal{F}(\rho_t)$ if $t\mapsto \rho_t$ is a geodesic starting at $\mu$ with vector field $t\mapsto v_t$. Considering $\psi\in C_c^{\infty}(\X)$ and the path $\rho_t$ from  $\mu$ to $(I+\nabla\psi)_{\#}\mu$ given by: $\rho_t=  (I+t\nabla\psi)_{\#}\mu$, for all $t\in [0,1]$, the Hessian of $\cF$ at $\mu$, $H\cF_{|\mu}$, is defined as a symmetric bilinear form on $C_c^{\infty}(\X)$ associated with the quadratic form
$\Hess_{\mu}\cF(\psi,\psi) := \frac{d^2}{dt^2}\Bigr|_{\substack{t=0}}\mathcal{F}(\rho_t)$. 

We now recall the formula of the Wasserstein Hessian of the (standard) Kullback-Leibler divergence (or relative entropy). 


\begin{proposition}\label{prop:hessian_kl} \citep[Section 9.1.2]{villani2021topics}.
Assume that $\tg$ has a density $\tg\propto e^{-V}$ where the  potential $V:X \to \R$ is $C^2(\X)$. 
The Hessian of~$\KL(\cdot|\tg)$ at $\mu$ is given, for any $\psi \in C_c^{\infty}(\X)$, by: 
\begin{multline}\label{eq:hessian_kl}
 \Hess_{\mu}\KL(\psi,\psi)\\ =
 \int  \left[ \ps{\hess_V(x)\nabla \psi(x), \nabla \psi(x)}
+ \| \hess \psi(x)\|^2_{HS} \right] \dd \mu(x)\\
 = \Hess_{\mu}\cG_V(\psi,\psi)+ \Hess_{\mu}\cU(\psi,\psi),
\end{multline}
where $\hess_V$ is the Hessian of $V$. 
\end{proposition}
The proof of \Cref{prop:hessian_kl} is provided in \Cref{sec:proof_hessian_KL} for completeness. The reader may also refer to \cite{korba2021kernel,duncan2019geometry} for similar computations on Wasserstein Hessians.

The KL divergence inherits the convexity of the target potential $V$ in the Wasserstein geometry. Indeed, if $\hess_V\succeq \lambda \ID$, then $\KL(\cdot|\tg)$ is $\lambda$-displacement convex, i.e. it is $\lambda$-convex along Wasserstein-2 geodesics, the underlying geometry for Wasserstein gradient flows. Yet, the Kullback-Leibler divergence is not a smooth objective in the Wasserstein sense, since its (Wasserstein) Hessian is not upper bounded even if the potential $V$ is smooth. Indeed, assume  $\hess_V\preceq M \ID$, i.e., the potential of the target distribution is $M$-smooth. This enables to control the first term in \eqref{eq:hessian_kl} by $M\|\nabla \psi\|^2_{L^2(\mu)}$, but the second term due to the negative entropy cannot be controlled similarly for any $\psi$ \cite{wibisono2018sampling,korba2020non}.

Hence in this context, it is not possible to prove a descent lemma along (Wasserstein) gradient descent for the KL, unless restricting to smooth directions \cite{korba2020non}. The non-smoothness of the KL is also the reason why many algorithms aiming to minimize the KL in the Wasserstein geometry rely on splitting-schemes such as the forward-backward algorithm, to perform a gradient descent (explicit) step on the potential energy part, and a JKO (implicit) step on the entropy part \cite{salim2020wasserstein,diao2023forward, domke2023provable}. In contrast, we will leverage the fact that the \emph{mollified} KL enjoys some smoothness properties that will allow us to derive a descent lemma in \Cref{sec:optim_guarantees}, at the price of loosing some convexity.

Still, we next show that $\cFe$ recovers displacement convexity (of the standard KL) as $\epsilon\to 0$, since its Hessian recovers the one of the KL. 
\begin{proposition}\label{prop:limiting_hessian_all}
Let $\mu \in \cP_2(\R^d)$. For any  $\psi \in C_c^{\infty}(\R^d)$, the Wasserstein Hessian of $\cFe$ converges to the one of the regular KL, i.e:
\begin{equation}
 \Hess_{\mu}\cFe(\psi,\psi)
 \xrightarrow[\varepsilon \to 0]{} \Hess_{\mu}\KL(\psi,\psi).
\end{equation}
\end{proposition}
The proof of \Cref{prop:limiting_hessian_all} can be found in \Cref{sec:hessian_mollified}; the main technical difficulties arise when dealing with the negative entropy term. This result shows that as $\epsilon \to 0$, one can recover the geometric properties of the KL. 

\Cref{prop:limiting_hessian_all} serves as an auxiliary finding within our study, not directly influencing other results, yet it enables us to illustrate key conceptual distinctions. Specifically, it demonstrates that while the standard Kullback-Leibler (KL) divergence is convex in the Wasserstein geometry for log-concave targets—exhibiting even strong convexity for targets that are strongly log-concave—it loses this convexity when mollified, although it gains smoothness with a positive $\epsilon$. This transition is typically delineated through lower and upper bounds on the Hessians within the Wasserstein framework. Getting a non-asymptotic, quantitative bounds on the Hessian of the mollified KL in terms of $\epsilon$ is the subject of future work. Such research could potentially offer insights into how small $\epsilon$ may be selected relative to the strong convexity constant of the target potential, ensuring the optimization objective maintains convexity.

\section{Optimization Guarantees}\label{sec:optim_guarantees}

We now turn to the analysis of the optimization error for VI in our setting, i.e. the optimization of $\cFe$.  
Under a smoothness assumption on the target potential, as well as moment conditions on the trajectory, one can obtain a descent lemma for the Wasserstein gradient descent iterates. 

\begin{assumption}\label{ass:smooth_potential}
The potential $V$ is $L$-smooth, i.e. for any $x,y\in \R^d$,
$\|\nabla V(x)-\nabla V(y)\|\le L\|x-y\|$.
\end{assumption}

\begin{assumption}\label{ass:descent_lemma}
$\mu_0$ is supported on $n$ Diracs, and the second moments of $(\mu_l)_{l\ge 0}$ are bounded by $\g>0$ along gradient descent iterations, i.e. $\int \|x\|^2 d\mu_l(x) <\g$ $, \forall l\ge 0$.
\end{assumption}
Bounded moment assumptions such as these are commonly used in
stochastic optimization, for instance in some analysis of the stochastic gradient descent \cite{moulines2011non}. We also verified empirically this assumption in a specific setting outlined afterwards. The target $\tg$ is a mixture of $100$ Gaussians that we approximate with a mixture of $10$ Gaussians. Then we run \eqref{eq:particle_gd} (equivalently  \eqref{eq:wgd}) for $1000$ iterations. 
 The expectations in  \eqref{eq:algoupdate} with respect to the Gaussian kernel are estimated by Monte Carlo with 100 samples. \Cref{fig:second_moment} displays the second moments of the particle distributions along iterations, for various dimensions. The 95\% confidence interval displayed in \Cref{fig:second_moment} is calculated based on 50 runs, and represents the randomness corresponding to Monte Carlo approximations, initialization of the target and initialization of our mixture. Our experiment shows that \Cref{ass:descent_lemma} holds for any dimension, i.e., the second moment of the particles distribution is bounded along the (discrete-time) flow. Further details on the setup are provided in \Cref{sec:numeric}. We now turn to one of our main results regarding the optimization of the mollified KL.

\begin{figure}[H]
    \centering
\includegraphics[width=0.6\columnwidth]{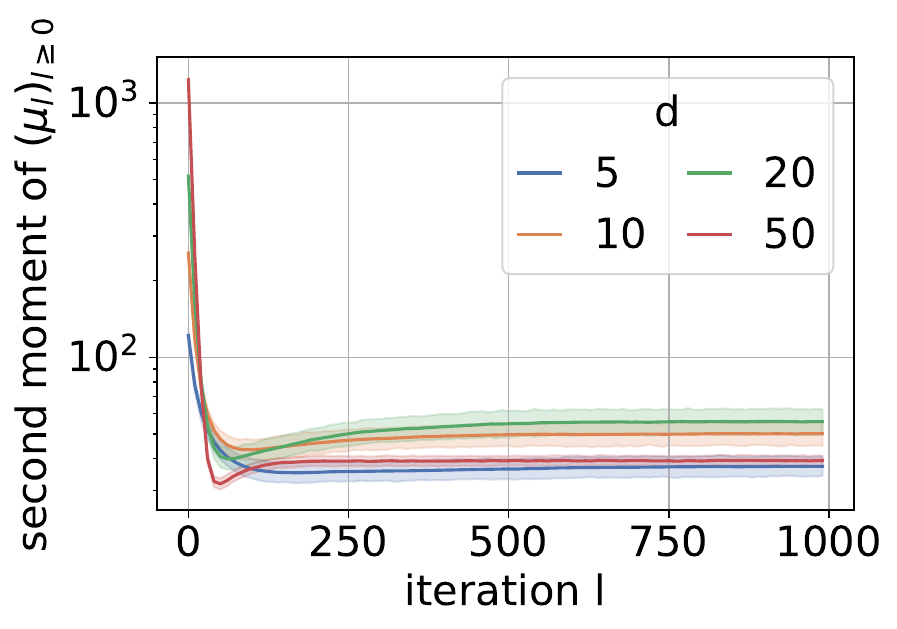}
    \caption{Second moment along Wasserstein gradient descent iterations.}
    \label{fig:second_moment}
\end{figure}

\begin{proposition}\label{prop:decreasing_functional}Suppose \Cref{ass:smooth_potential} and  \Cref{ass:descent_lemma} hold. Consider the sequence of iterates of Wasserstein gradient descent of $\cF_{\epsilon}$  defined by \eqref{eq:wgd}. Then, the following inequality holds:
\begin{align*}
\cFe(\mu_{l+1})-\cFe(\mu_l)\leq -\gamma \left(1-\frac{\gamma}{2}M \right) \|\nabla \cFe'(\mu_l)\|^2_{L^2(\mu_l)}.
\end{align*}
where $M= L + K_{\epsilon,n,\g}$, and $K_{\epsilon,n,\g}$ is a constant depending on $\epsilon, n ,h$. 
\end{proposition}
Hence, for a  small enough step-size $\gamma$, the latter proposition shows that the objective decreases at each iteration. We now provide a proof for this result,
using similar techniques as \cite{arbel2019maximum,korba2020non}. 
The main technical difficulties are left in the appendix  and are related to showing the descent for the mollified entropy part, see \Cref{sec:proof_prop_decreasing} for details. 
\begin{proof}[Proof of \Cref{prop:decreasing_functional}]\label{proof:prop:decreasing_functional}
Consider a path between $\mu_l$ and $\mu_{l+1}$ of the form $\rho_t	=(\psi_t)_{\#}\mu_l$ with $\psi_t=(\Id+ t\nabla \cFe'(\mu_l))$. We have $\frac{\partial \rho_t}{\partial t}= \div(\rho_t v_t)$ with $v_t=-\nabla \cFe'(\mu_l)\circ \psi_t^{-1}$. The latter continuity equation holds in the sense of distributions \citep[Chapter 8]{ambrosio2008gradient} and holds for discrete measures. 
The function 
$t\mapsto\cFe(\rho_t)$ is differentiable and hence absolutely continuous. Therefore one can write:
\begin{multline}\label{eq:taylor_expansion_decreasing}\cFe(\rho_{\gamma})=\cFe(\rho_0) + \gamma \dv[]{}{t}\bigg|_{t=0}\cFe(\rho_t)\\+  \int_0^{\gamma}\left[ \dv[]{}{t}\mathcal{\cFe}(\rho_t)- \dv[]{}{t}\bigg|_{t=0}\cFe(\rho_t)\right] dt.
\end{multline}
Moreover, using the chain rule in the Wasserstein space, we have successively:
\begin{multline}
\dv[]{}{t}\cFe(\rho_t) = \ps{\nabla \cFe'(\rho_t), v_t}_{L^2(\rho_t)}, \\\text{ and } \dv[]{}{t}\bigg|_{t=0}\cFe(\rho_t)  =-\|\nabla \cFe'(\mu_l)\|^2_{L^2(\mu_l)}.
\end{multline}
Then, since $\cFe = \cUe+\cGe$, we have first under \Cref{ass:smooth_potential} that $\ke \star V$ is $L$-smooth and by \Cref{prop:descent_potential} that:
\begin{equation}
 \dv[]{}{t} \cGe(\rho_t) - \dv[]{}{t} \cGe(\rho_t) \Big\vert_{t=0} \leq L \; t \|\nabla \cFe'(\mu_l)\|^2_{L^2(\mu_l)},
\end{equation}
and by \Cref{prop:lipschitz_mollified_entropy} and \Cref{ass:descent_lemma}:
\begin{equation*}
    \frac{d}{dt} \cUe(\rho_t) - \frac{d}{dt} \cUe(\rho_t) \Big\vert_{t=0} \leq  K_{\epsilon,n,h} t \Vert \phi \Vert^2_{L^2(\mu_l)},
\end{equation*}
where $K_{\epsilon,n,\g}= 
   \nicefrac{1}{\epsilon^2} + \nicefrac{2\sqrt{\g n}}{ \epsilon^3} + \nicefrac{\sqrt{n}}{\epsilon^2} + \nicefrac{n\sqrt{\g}}{2 \epsilon^3}$.
Hence, the result follows directly by applying the above expressions to \Cref{eq:taylor_expansion_decreasing} where $M = L+ K_{\epsilon,n,\g}$.
\end{proof}
As a corollary, we obtain the convergence of the average of squared gradient norms along iterations.
\begin{corollary}\label{cor:average_gradient} Let $c_{\gamma}=\gamma(1-\frac{\gamma M}{2})$. Under the assumptions of \Cref{prop:decreasing_functional}, one has
\begin{equation}
\frac{1}{L}\sum_{l=1}^L \|\nabla \cFe'(\mu_l)\|^2_{L^2(\mu_l)}\le \frac{\cFe(\mu_0)}{2 c_{\gamma} L}.
\end{equation}
\end{corollary}

In contrast with the KL that is non-smooth as explained in \Cref{sec:kl_non_smooth}, the mollified KL is smooth, which is why we can prove the descent lemma in  \Cref{prop:decreasing_functional} and the rate on average gradients. The descent lemma and its corollary imply that the sequence of squared gradient norms is summable and hence converges to zero.

We illustrate the validity of the rate derived in \Cref{cor:average_gradient} with simple experiments. The variational family there is a family of Gaussian mixtures with $10$ components, while the target is a Gaussian mixture with $100$ components. 
\Cref{fig:average_gradient} shows the convergence of the cumulative sum $\frac{1}{L} \sum_{l=1}^{L} \|\nabla \cFe'(\mu_l)\|^2_{L^2(\mu_l)}$ along iterations, for various dimensions and in log scale. Similarly to the previous experiment, the expectations involved in the gradient descent schemes are estimated using Monte Carlo with 100 samples. The 95\% confidence interval displayed in \Cref{fig:average_gradient} is computed based on 50 runs, representing the randomness due to Monte Carlo approximations, randomization of the target and ininital distribution for the scheme. The term $\|\nabla \cFe'(\mu_l)\|^2_{L^2(\mu_l)}$ also involves expectations that are estimated by Monte Carlo with 1000 samples. \Cref{fig:average_gradient} illustrates that the cumulative sum is indeed of order $\frac{1}{L}$ as stated by \Cref{cor:average_gradient}. A detailed description of the experimental setup can be found in \Cref{sec:numeric}.

\begin{figure}[H]
    \centering
\includegraphics[width=0.6\columnwidth]{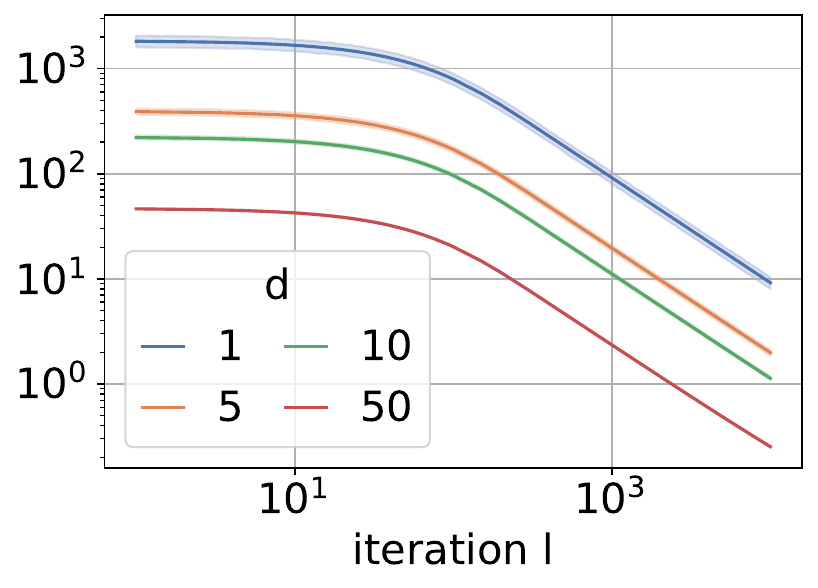}
    \caption{Illustration of the rate of $\frac{1}{L} \sum_{l=1}^{L} \|\nabla \cFe'(\mu_l)\|^2_{L^2(\mu_l)}$ derived in Corollary \ref{cor:average_gradient}}    \label{fig:average_gradient}
\end{figure}

\begin{remark}
    Non-convex rates similar to our \Cref{cor:average_gradient} have been obtained for Langevin Monte-Carlo \cite{balasubramanian2022towards} or Stein Variational Gradient Descent (SVGD) algorithm \cite{korba2020non} leveraging similar techniques and smoothness of the potential. However, since Langevin Monte Carlo and SVGD optimizes the (standard) KL divergence, the squared gradient norm correspond to the Fisher Divergence and Kernel Stein Discrepancy respectively, that are valid probability divergences. In our setting, \Cref{cor:average_gradient} implies the following. 
    If $\mu_l$ converges weakly to some distribution $\mu_{\infty}$ (up to a subsequence) as $l\to \infty$, the Wasserstein gradient of $\cFe$ given in \Cref{eq:wgd_cFe} is zero on the support of  $\mu_{\infty}$, assuming $\mu \mapsto \|\nabla\cFe'(\mu)\|^2_{L^2(\mu)}$ is lower semi continuous with respect to the weak topology of measures.  
    This can be rewritten $\nabla \ke \star (\log(\ke \star \mu_{\infty}) - \log (\tg))=0$ $\mu_{\infty}$-a.e, i.e. $\ke \star (\log(\ke \star \mu_{\infty}) - \log (\tg))=c$ $\mu_{\infty}
    -$a.e. for some constant $c$. 
   
\end{remark}

\section{Approximation Guarantees}\label{sec:approx_guarantees}

In this section, we investigate the approximation accuracy of a finite mixture of Gaussians to the posterior, i.e. minimizers of the objective functional $\cFe$ (assuming we are able to find these minimizers, e.g. after optimization). We obtain non-asymptotic rates with respect to the number of components in the mixture. For ease of notation, we will denote by $\ke^x := \ke(\cdot - x)$ for any $x\in \R^d$.
We first consider the following assumption on the target distribution. 

\begin{assumption}\label{ass:well_posedness_posterior}
The target posterior distribution $\tg$ has a mixture representation form, i.e. there exists $P$ on $\R^d$ such that
\begin{equation*}
\tg = \int_{\Theta} 
\ke^w
dP(w).
\end{equation*}
\end{assumption}

Notice that \Cref{ass:well_posedness_posterior} is a relatively weak assumption, as mixture of Gaussians are dense in the space of probability distributions \cite{delon2020wasserstein}. We now state our second main result.

\begin{theorem}\label{th:kl_quantization} Suppose \Cref{ass:well_posedness_posterior} holds and define 
\begin{equation}
C_{\tg}^2= \int \frac{\int 
(\ke^m(x))^2dP(m)}{\int 
\ke^w(x)
dP(w)} dx.
\end{equation}
Define  
$\C_n = \left\{\ke \star \mu_n,\; \mu_n\in \cP_n(\R^d)\right\}$, where $\cP_n(\R^d)$ is the set of discrete probability distributions supported on $n$ Dirac masses.  
Then,
\begin{equation*}
\min_{\mu_n \in \cP_n(\R^d)} \KL(\ke \star \mu_n| \tg)
\leq C_{\tg}^2 \frac{\log(n) + 1}{n}.
\end{equation*}
\end{theorem}
Our result is novel and quantifies the approximation quality of the family of mixtures of $n$ Gaussian distributions (with equal weight and constant covariance) in the (reverse) Kullback-Leibler sense.

    A major limitation in the use of Gaussian distributions in VI arises from the inherent simplicity of this family. In particular, the unimodality of the Gaussian distribution becomes a critical stumbling block when the target distribution is multimodal. 
    A notable exception exists in the work of Katsevich \& Rigollet (2023), which provides an error bound for cases where the target is a posterior distribution in the Bayesian inference context. As the sample size goes to infinity, the Bernstein Von-Mises theorem shows that the posterior distribution asymptotically converges to a Gaussian distribution, thereby lending some predictability to the approximation error in this specific scenario.
In stark contrast,  Theorem \ref{th:kl_quantization} offers a more versatile result, applicable to any target distribution, including those encountered in Bayesian inference with a fixed sample size. It shows that increasing the number of components in a Gaussian mixture can significantly mitigate the limitations of Gaussian VI. As we expand the mixture, the approximation error not only decreases, it converges to zero. This result highlights the potential of complexifying the variational family to achieve more accurate approximations of the target distribution.


The proof of \Cref{th:kl_quantization} follows the steps of \cite{li1999mixture}, that proved similar guarantees for the forward KL (akin to likelihood maximization), while we focus on the reverse KL, i.e. the one considered in variational inference. Hence our proof requires non-trivial different inequalities and intermediate lemmas that are deferred to \Cref{sec:technical_lemmas_quantization}.

\begin{proof}[Proof of \Cref{th:kl_quantization}]
We will prove the previous result by induction.
We denote by $\nu_n$ the minimizer of the Kullback-Leibler divergence to the target within this family, i.e., 
\begin{equation*}
\nu_n := \argmin_{\mu_n \in \cP_n(\R^d)} \KL(\ke \star \mu_n| \tg),
\end{equation*}
and $D_n = \KL(\ke \star \mu_n| \tg)$.
For any $m \in \R^d$, we consider the distribution $\rho_{n+1}^m \in \C_{n+1}$ defined as
\begin{equation*}
\rho_{n+1}^m
= (1 - \alpha) (\ke \star \mu_n) + \alpha \ke^m
\end{equation*}
where $\alpha = \nicefrac{1}{n+1}$. 
Therefore we have
$   D_{n+1} = \KL(\ke \star \mu_{n+1}| \tg) \leq \KL(\rho_{n+1}^m| \tg)$.
By definition of the Kullback-Leibler divergence, denoting $f(x) = x \log x$, we have
\begin{align*}
\KL(\rho_{n+1}^m| \tg) 
= \int f(r_{n+1})d\tg\;,
\end{align*}
where we define $r_{n+1}$ and $r_0$ as:
\begin{align*}
r_{n+1} := \frac{\rho_{n+1}^m}{\tg} &= (1 - \alpha) \frac{(\ke \star \mu_n)}{\tg} + \alpha \frac{\ke^m}{\tg}
:=r_0+ \alpha \frac{\ke^m}{\tg}.
\end{align*} 
Define $ B(x) =(x \log x - x + 1)/(x - 1)^2$ for $x\in [0,+\infty[.$ Note that $r_{n+1}(x) \geq r_0(x) $ for any $x$, then using that $B$ is decreasing (see  \Cref{lem:decreasing}), we have $B(r_{n+1}(x))\le B(r_0(x))$. It follows that
\begin{align}
& r_{n+1} \log(r_{n+1}) \nonumber \\
& \quad \leq r_{n+1} - 1 + B(r_0)(r_{n+1} - 1)^2 \nonumber \\
& \quad = r_0 + \alpha \frac{\ke^m}{\tg} - 1 + B(r_0)\left(r_0 + \alpha \frac{\ke^m}{\tg} - 1\right)^2 \nonumber\\
& \quad = r_0 + \alpha \frac{\ke^m}{\tg} - 1 \nonumber\\
&\quad \quad + B(r_0)\big\{(r_0- 1)^2  + \left(\alpha \frac{\ke^m}{\tg}\right)^2 + 2 \alpha (r_0 - 1) \frac{\ke^m}{\tg}\big\}\nonumber\\
& \quad = \alpha \frac{\ke^m}{\tg} + r_0 \log(r_0) + \left(\alpha \frac{\ke^m}{\tg}\right)^2 B(r_0) \nonumber\\
&\quad \quad + 2 \alpha B(r_0)(r_0 - 1) \frac{\ke^m}{\tg}.\label{eq:inequality_from_B}
\end{align}

Moreover, we have the following inequality:
\begin{align*}
D_{n+1} &= \int D_{n+1} dP(m) \\
& \leq \int \KL(\rho_{n+1}^m| \mu) dP(m)
\\
& = \alpha + \int r_0(x) \log(r_0(x)) d\tg(x) \\
& \quad+ \alpha^2 \iint\frac{\ke^m(x)^2}{\tg(x)^2} B(r_0(x)) d\tg(x) dP(m)\\
& \quad+ 2 \alpha \int B(r_0(x))(r_0(x) - 1) d\tg(x),
\end{align*}
where we used \Cref{eq:inequality_from_B} in the last equality. We now focus on bounding each term on the r.h.s. of the previous inequality.
By definition of $r_0$, the second term can be rewritten
\begin{equation*}
\int r_0 \log(r_0) d\tg = (1 - \alpha) \log(1 - \alpha) + (1 - \alpha) D_n.
\end{equation*}
We now turn to the third term. For any $x \in \R^{+}$, since $B$ is monotone decreasing, $B(r_0(x))\le B(0) = 1$. Under \Cref{ass:well_posedness_posterior}
, it follows that 
\begin{multline*}
\int \int \frac{\ke^m(x)^2}{\tg(x)^2} B(r_0(x)) d\tg(x) dP(m) \\\leq \int \int \frac{\ke^m(x)^2}{\tg(x)} dP(m)dx 
= C_{\tg}^2.
\end{multline*}
Finally let's focus on the last term. We have $B(x) (x-1)\leq \sqrt{x} - 1$ using \Cref{lem:bound}. Denoting $H^2(f, g) = 1 - \int \sqrt{f(x) g(x)} dx \in [0,1]$ the squared Hellinger distance between $f$ and $g$, we have
\begin{multline*}
\int B(r_0)(r_0 - 1) d\tg \quad \leq \int (\sqrt{r_0} - 1) d\tg\\
\quad= \sqrt{1 - \alpha}(1 - H^2(\ke \star \mu_n, \tg)) - 1\quad \leq \sqrt{1 - \alpha} - 1.
\end{multline*}
Finally, we have
\begin{multline*}
D_{n+1} \leq \alpha + (1 - \alpha) \log(1 - \alpha) + (1 - \alpha) D_n + \\
\alpha^2 C_{\tg}^2 + 2 \alpha (\sqrt{1 - \alpha} - 1) 
\leq (1 - \alpha) D_n + \alpha^2 C_{\tg}^2,
\end{multline*}
where the last inequality uses that $- \alpha + (1 - \alpha) \log(1 - \alpha) + 2 \alpha \sqrt{1-\alpha} \leq 0$ (see \Cref{lem:inequality_alpha}).

Now, recall that $\alpha = 1/(n+1)$. Denoting $U_n = n D_n$, our previous computations imply that $U_{n+1} \leq U_n + \nicefrac{C_{\tg}^2}{n+1}$, which by telescoping yields $U_n - U_0 \leq C_{\tg}^2 H_n$, 
where $H_n$ denotes the harmonic number and is upper bounded by $1 + \log(n)$. The rate on $D_n$ follows.
\end{proof}

Our result is analog to the one of \citet{li1999mixture} that bounds the forward Kullback-Leibler divergence to the target. Indeed under  \Cref{ass:well_posedness_posterior}, their Theorem 1 states that 
\begin{equation}
\argmin_{\mu_n \in \cP_n(\R^d)} \KL(\tg|\ke \star \mu_n) \le \frac{C_{\tg}^2 h}{n}
\end{equation}
where $h= 4 \log(3 \sqrt{e}+a)$ is a constant depending on $\epsilon$ since $a = \sup_{m_1, m_2 \in \R^d} \log \left(\ke^{m_1}(x)/\ke^{m_2}(x)\right)$. In our case, the constant in the rate does not involve $h$ as we do not rely on the same functions (our $B\le1$ instead of $B\le h$ in their case).

When \Cref{ass:well_posedness_posterior} does not hold, they also show in Theorem 2 that for every $g_{P}=\int \ke(\cdot-w)dP(w)$,
\begin{equation}
\argmin_{\mu_n \in \cP_n(\R^d)} \KL(\tg|\ke \star \mu_n) \le  \KL(\tg|g_P) +\frac{C_{\tg,P}^2 h}{n}
\end{equation}
where $C_{\tg,P}^2 = \int \frac{\int \ke^m(x)^2 dP(m)}{(\int \ke^w(x) dP(w))^2} d\tg(x)$. However, they can easily obtain this result as a consequence of their first theorem along with the linearity of the forward KL. In constrast, the reverse KL does not verify linearity nor triangular inequality hence we cannot obtain readily such a generalization.

Notice that the forward KL rate obtained in \cite{li1999mixture} is of order $1/n$, outpacing the one we attained. This is due to our chosen variational family, which is a Gaussian mixture with fixed weights. However, considering non-fixed weights  (i.e. non-equally weighted mixtures) allows us to set $\alpha = 2/(n+1)$, thus achieving the exact same rate as \cite{li1999mixture} for the reverse KL.

Since the Total Variation can be written as an Integral probability metric over measurable functions $f:\R^d\to [-1,1]$, we deduce from Pinsker's inequality and \Cref{th:kl_quantization} that a minimizer $\mu_n$ of $\KL(\ke \star \cdot|\tg)$ achieves the following bound for the integral approximation error among this set of functions:
\begin{equation*}
 \left|\int\!fd(\ke \star\mu_n) - \int\! fd\tg\right| \le  \sqrt{ \frac{C_{\tg}^2(\log(n) + 1)}{2n}.}
\end{equation*} 
The latter is then comparable to the integral approximation error of MCMC methods which is known to be of order $\mathcal{O}(n^{-\frac{1}{2}})$ when using $n$ particles \cite{latuszynski2013nonasymptotic}. 

\begin{figure}[H]
    \centering
    \includegraphics[width=0.60\columnwidth]{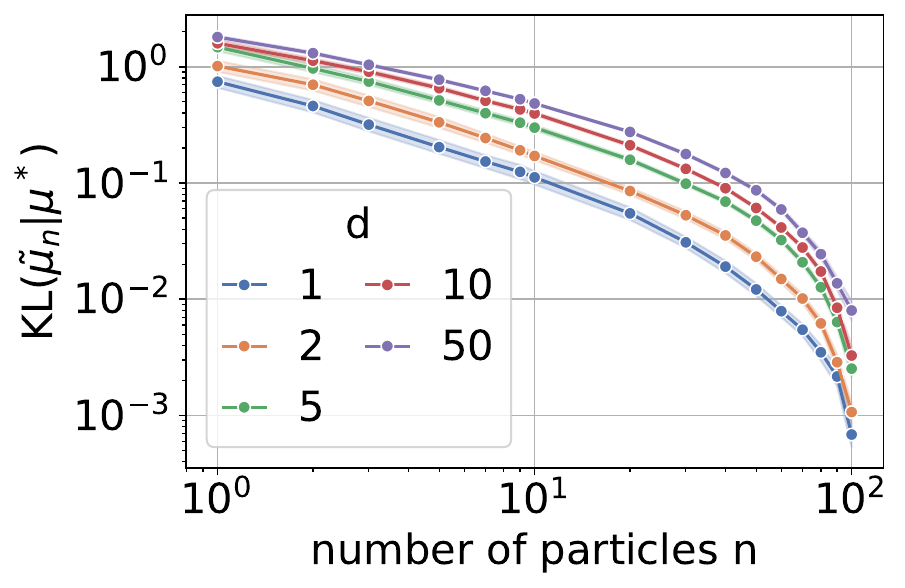}
    \caption{Illustration of the rates of \Cref{th:kl_quantization}, where $\nu_n=\argmin_{\nu \in \C_n}\KL(\nu| \mu^{\star})$ is approximated by $\tilde{\nu}_n$.}
    \label{fig:quantization}
\end{figure}


We finally test numerically the validity of \Cref{th:kl_quantization} in a simple setting. The target distribution considered is a Gaussian mixture with $100$ components. We denote by $(x_i^{\star})_{i \leq 100}$ the mean of these components. For any $n \in [1, 100]$, the objective is to solve \eqref{eq:vi}  and find $\nu_n :=  \argmin_{\nu \in \C_n}  \KL(\nu|\tg)$, where $\mathcal{Q}_n$ represents the family of Gaussian mixtures with $n$ components. This minimizer is approximated by selecting only the first $n$ components $(x_i^{\star})_{i \leq n}$ of $\tg$, and we denote $\tilde{\nu}_n$ the resulting approximate distribution. Note that in that specific setting, the variational family $\C_n$ and the target distribution $\tg$ share the same standard deviation. \Cref{fig:quantization} shows the convergence rate of $\KL(\tilde{\nu}_n | \tg)$ with respect to the number of components $n$, for various dimensions. The objective 
is estimated by Monte Carlo with 1000 samples. The 95\% confidence interval displayed in \Cref{fig:quantization} is approximated based on 100 samples, representing the randomness corresponding to Monte Carlo approximation of the KL, and the initialization of the target. \Cref{fig:quantization} illustrates that the Kullback-Leibler divergence between $\tilde{\nu}_n$ and the $\tg$ is indeed decreasing linearly with $n$.  This result proves the validity of the rates derived in \Cref{th:kl_quantization} for this specific setting. A full description of the experimental setup can be found in \Cref{sec:numeric}.


\section{Related work}\label{sec:related_work}

In this section we discuss relevant related work.

\emph{Theoretical guarantees for Variational Inference.} For the variational inference optimization problem in \eqref{eq:vi}, frequently employed constraint sets $\C$ in existing literature encompass the set of non-degenerate Gaussian distributions, location-scale families, 
mixtures of Gaussian components, and the set of product measures. In the Gaussian setting, \cite{lambert2022variational,diao2023forward} have been the first to leverage the geometry of Wasserstein gradient flows to study the convergence properties of variational inference, and provide convergence rates when the target $\mu^*\propto e^{-V}$ has a smooth and strongly convex potential $V$. In Mean-Field Variational inference (MFVI), the space $\C$ in \eqref{eq:vi} is taken to be the
class of product measures over $\R^d$, written $\cP(\R)^{\otimes d}$. Several works have proposed algorithms in this context via wasserstein gradient flows \cite{yao2022mean,lacker2023independent}. \cite{jiang2023algorithms} consider a smaller subset of $\C$, namely a polyhedral subset for which they can derive optimization and approximation guarantees. However the previous work do not tackle mixture of Gaussians for the variational family. 



\emph{Mollified Relative entropies.} A closely related line of work to this paper is the one of 
\cite{carrillo2019blob,craig2023blob,craig2023nonlocal,carrillo2023nonlocal} that study Wasserstein gradient flows of mollified relative entropies and the associated particle systems, that are of particular interest in the literature of partial differential equations and kinetic theory. 
In \cite{carrillo2019blob}, the authors mention the mollified negative entropy $\cUe$ that we define in \eqref{eq:mollified_e} as a regularization of the negative entropy $\cU(\mu)=\int \log(\mu)d\mu$ (or entropy of order 1), but they do not study it. Instead, they focus on  a closely related functional, defined as $\tilde{\cUe}(\mu)=\int \log(\ke\star \mu) d\mu$ (i.e. with only one convolution inside the logarithm, while $\cUe$ involves two convolutions). While they mention the possible choice of $\cUe$ as a regularization of the entropy $\cU$, they choose to study the alternative regularization $\tilde{\cUe}(\mu)$ for numerical reasons, as the Wasserstein gradient of the latter functional writes as an integral over the distribution of the particles, while the one of $\cUe$ (hence $\cFe$) writes as an integral over the whole space w.r.t. Lebesgue measure, as explained in \Cref{sec:algorithm}. Hence their results on $\lambda$-convexity\footnote{$\lambda$-convexity for $\lambda\ge0$ recovers displacement convexity, while $\lambda\le 0$ recovers smoothness.} of the functional or the particle system differ from our setting. \cite{craig2023blob} focus on a mollified chi-square divergence that corresponds to a weighted second order entropy; \citep{li2022sampling} studies another mollified approximation of the chi-square divergence. 
\cite{carrillo2023nonlocal} also study $\lambda$-convexity of entropies but only for entropies of order strictly greater than 1. Finally \cite{craig2023nonlocal} study functionals of the form $\int f_{\epsilon}(\ke \star \mu)d\cL$ as approximations of $\int f(\mu)d\cL$ where $\cL$ denotes the Lebesgue measure. In their case $f_{\epsilon}$ is a specific function depending on $f$ and $\epsilon$, which excludes $\cUe$ and thus also differs from our setting.

\emph{VI on mixtures.} Several works have tackled VI on mixtures on a computational aspect.  
\citet{gershman2012nonparametric} optimize (with L-BFGS, that is a quasi Newton method) an approximate ELBO (recall that the ELBO is the reverse KL we consider up to an additive constant), using several consecutive approximations of ELBO terms for the case of mixture of Gaussians. In the end, their optimization objective differs a lot from the original KL objective from VI, that is a valid divergence between probability distributions - in contrast with their objective.
\citet{arenz2018efficient} adopt an Expectation-Maximization (EM) approach. As noted in \citep{aubin2022mirror,kunstner2021homeomorphic} EM can be seen as mirror descent scheme on the KL. Also, this algorithm can be seen as an Euler discretization of the gradient flow of the KL in the Fisher-Rao geometry \citep{domingo2023explicit,chopin2023connection}. The parallel can be seen from eq (5) or (8) in \citep{arenz2018efficient}, that take a similar form as eq (6) in \citep{chopin2023connection}, i.e. a geometric update on the distributions, i.e. that act directly on updating densities (in a "vertical" manner), equivalently weights. In contrast, we focus on gradient descent dynamics, that correspond to a time discretization of the KL gradient flow in the Wasserstein geometry. This correspond to "horizontal" updates, where particles are displaced at each iteration.
\citet{lin2019fast} use natural gradient updates for VI in the natural parameter space (e.g. means for Gaussians). However, from \citep{raskutti2015information,kunstner2021homeomorphic}, it is known that this is equivalent to mirror descent on the exponential family parameters, which again is related to Fisher-Rao dynamics on the space of probability distributions (see eq (13) in \citep{chopin2023connection}).

\section{Conclusion}

The goal of this paper is to improve our theoretical understanding of variational inference algorithms in the non-Gaussian case.  We consider here a specific family of distributions, a mixture of Gaussians with constant covariance and equally weighted components, that enables us to derive novel results for the approximation and optimization error for Variational Inference. We derive theoretical guarantees regarding gradient descent of the objective (i.e. a descent lemma proving that the objective decreases at each iteration) leveraging smoothness of the objective and the Wasserstein geometry. We also derive novel approximation results for minimizers of the objective. 


In our study, we chose to simplify our exploration of Variational Inference (VI) within the context of Gaussian Mixtures by assuming uniform weights for each Gaussian component and by fixing the covariances. Extending our findings to more complex scenarios, where the weights of each Gaussian are dynamically optimized and the covariances are variable, represents a significant challenge that goes beyond the scope of our current research. For instance, the task of optimizing the weights attached to each Gaussian component introduces a shift from the Wasserstein dynamics, which are central to our current discussion, to Fisher-Rao dynamics.  Achieving a counterpart to our current optimization result \Cref{prop:decreasing_functional} under these conditions would not only require the adoption of alternative proof techniques but also a deep dive into the intricacies of Fisher-Rao dynamics, which diverge significantly from those of Wasserstein. Furthermore, the optimization of covariance matrices introduces another level of complexity. Such an endeavor requires a unique analytical framework, primarily due to the constraints imposed by the requirement that these matrices be positive definite. While this aspect of the analysis is crucial for a comprehensive understanding of VI in Gaussian mixtures, it requires a specialized approach that our current methodology does not cover. 
The exploration of dynamic weight optimization and variable covariance matrices within the context of Gaussian Mixtures in VI presents a rich avenue for future work.

\section*{Impact statement}

Variational inference is a crucial tool for modern and large-scale Bayesian inference, as it can approximate complex posterior distributions in a computationally efficient manner.  However, its theoretical properties are poorly understood outside the setting where the variational family is a set of Gaussians. Studying the theoretical foundations of variational inference is essential for understanding the method's limitations, and guiding users in making informed choices about model assumptions and optimization strategies. 


\section*{Acknowledgements}

The authors are grateful for Irène Waldspurger, José Carrillo and Marc Lambert for helpful discussions. A.K and E.M acknowledge the support of ANR Chaire AI, SCAI project. T.H acknowledge the support of ANR-CHIA-002, ”Statistics, computation and Artificial Intelligence”. Part of the work has been developed under the auspice of the Lagrange Center for Mathematics and Calculus. This work was granted access to the HPC resources of IDRIS under the allocation AD011013313R2 made by GENCI (Grand Equipement
National de Calcul Intensif).

\bibliography{biblio}
\bibliographystyle{icml2024}

\newpage
\appendix
\onecolumn

The appendix is organized as follows. \Cref{sec:particle_implementation} details the computations for the Wasserstein gradients and the particle scheme corresponding to the optimization of the mollified relative entropy. \Cref{sec:MOG_lambert} discusses the connection with the algorithm and framework presented in \cite{lambert2022variational}. \Cref{sec:technical_lemmas_quantization} contains the intermediate lemmas needed for the proof of \Cref{th:kl_quantization}. \Cref{sec:wasserstein_hessians} contains the proofs of the Propositions regarding Wasserstein Hessians. \Cref{sec:numeric} outlines the setup used for the numerical experiments.

\section{Particle implementation of the gradient flow}\label{sec:particle_implementation}

A solution of \eqref{eq:wgf} is implemented by the Mac-Kean Vlasov process:
\begin{equation}\label{eq:ode}
\dot{m}_t = - ( \nabla \cUe'(\mu_t)(m_t)+ \nabla \cGe'(\mu_t) (m_t)).
\end{equation}
Here we detail the computation of the vector field in  \eqref{eq:ode} and its particle implementation. \\

\noindent For the negative entropy part, we can rewrite $\cUe(\mu)=\int U(\ke\star \mu(\theta))d\theta$ where $U:x\mapsto x\log(x)$. We have that 
\begin{equation*}
\cUe'(\mu)(\cdot) = \ke\star (U'\circ (\ke\star \mu))(\cdot)= \int_{\R^d} \ke(\theta- \cdot) U'\left(\int \ke(\theta -y) d\mu(y))\right) d\theta
\end{equation*}
where  $U':x\mapsto \log(x)+1$. Hence, computing $\cUe'$ requires an integration over $\R^d$.  Then, we have, since $\ke$ is smooth, using an integration by parts with $\nabla U'(x)=\nabla \log(x)$ and symmetry of $\ke$, for any $w\in \R^d$ we have:
\begin{align}
\nabla_w\cUe'(\mu)(w) &= \nabla_w\ke\star (U'\circ (\ke\star \mu))(w)= \int_{\R^d} \nabla_w \ke(\theta- w) U'\left(\int \ke(\theta -y) d\mu(y)\right) d\theta \nonumber\\
 &= - \int_{\R^d} \nabla_{\theta} \ke(\theta- w) U'\left(\int \ke(\theta -y) d\mu(y)\right) d\theta \nonumber \\\nonumber
&= + \int_{\R^d}  \ke(\theta-w) \nabla_{\theta} U'\left(\int \ke(\theta -y) d\mu(y)\right) d\theta\\\nonumber
&= \int_{\R^d}  \ke(\theta- w) \nabla_{\theta} \log \left(\int \ke(\theta -y) d\mu(y)\right) d\theta\\
& = \int_{\R^d}  \ke(\theta- w) \frac{ \int \nabla \ke(\theta -y) d\mu(y)}{\int \ke(\theta -y) d\mu(y)} d\theta.
\label{eq:wgd_mollified_u}
\end{align}
Finally, if $\mu_t $ is an atomic measure of the form $\mu_t = \frac{1}{N}\sum_{i=1}^N \delta_{m_t^{(i)}}$, then a particle implementation of \eqref{eq:ode2} reduces to solving a system of ordinary differential
equations for the locations of the Dirac masses:
\begin{equation}\label{eq:particle_ode}
\dot{m}_t^{(j)}=  - \int_{\R^d} \nabla V(y)\ke(y-m_t^{(j)})dy - \int_{\R^d} \frac{\sum_{i=1}^N \nabla \ke(y-m_t^{(i)})}{\sum_{i=1}^N  \ke(y-m_t^{(i)})}\ke(y-m_t^{(j)})dy.
\end{equation}

\noindent For the potential energy part, we can rewrite
\begin{multline*}
\cGe(\mu)=\int_{\R^d} V(\theta) d(\ke\star \mu)(\theta)= \int_{\R^d} V(\theta) \int \ke(\theta-m)d\mu(m)d\theta\\ = \iint_{\R^d} \ke(\theta-m)V(\theta)d\theta d\mu(m):= \int_{\R^d} \tVe(m) d\mu(m),
\end{multline*}
where $\tVe(m) = \int_{\R^d} \ke(\theta-m)V(\theta)d\theta = \ke \star V(m)$.
Hence, we have successively for any $w\in \R^d$
\begin{align*}
\cGe'(\mu)(w) &= \tVe(w), \\
\nabla_w \cGe'(\mu)(w) &= \nabla_w \tVe(w) = \int_{\R^d} \nabla_w\ke(\theta-w)V(\theta)d\theta = - \int_{\R^d} \nabla_\theta\ke(\theta-w)V(\theta)d\theta = \int_{\R^d} \ke(\theta-\cdot)\nabla V(\theta)d\theta 
\end{align*}
using again an integration by parts. Hence, \eqref{eq:ode} becomes:
\begin{equation}\label{eq:ode2}
\dot{m}_t = - \int_{\R^d} \ke(\theta-m_t)\nabla V(\theta)d\theta  - \int_{\R^d}  \ke(\theta- m_t) \frac{ \int \nabla \ke(\theta -y) d\mu_t(y)}{\int \ke(\theta -y) d\mu_t(y)} d\theta.
\end{equation}

\section{Mixture of Gaussians optimization}\label{sec:MOG_lambert}

\citet{lambert2022variational} consider a Gaussian approximation of the Langevin diffusion given by Saarka's heuristic, i.e. $X_t \sim \mu_t$ where $\mu_t$ is the solution of Fokker-Planck equation is replaced by $Y_t\sim \mathcal{N}(m_t,\Sigma_t)$ where
\begin{align*}
\dot{m}_t &= - \E[\nabla V(Y_t)] \\
\dot{\Sigma}_t & = 2 \Id -\E[\nabla V(Y_t) \otimes (Y_t-m_t)+(Y_t-m_t)\otimes \nabla V(Y_t) ]
\end{align*}
They prove that the law of $Y_t$ is the gradient flow of the KL on the Bures-Wasserstein manifold $\BW(\R^d)\cong \R^d \times S_d^{++}$ (the space of Gaussians equipped with the $W_2$ distance); which is a submanifold of $\cP_2(\R^d)$. It can be seen as "Projected WGF" where the Wasserstein gradient of the KL is projected onto the tangent space of the submanifold; another way to view it is to see that its the GF of the KL on the Bures-Wasserstein manifold.\\

\noindent Then, they propose to write a Gaussian mixture $\rho$ on $\R^d$ as $\rho_{\nu}(\theta) = \int_{\BW(\R^d)} p(\theta) d\nu(p)$\footnote{We can rewrite it as $ \int_{\R^d\times S_{d}^{++}} p_{y,\Sigma}(\theta) d\nu(y,\Sigma)$} where $\nu$ is a measure over $\BW(\R^d)$; hence $MOG$ is isomorphic to $\cP_2(\BW(\R^d))$. Then the WGF of $\nu \mapsto \KL(\rho_{\nu}|\tg)$, ie the GF of this functional over $\cP_2(\BW(\R^d))$ is implemented through a particle system $\nu_t=\frac{1}{N}\sum_{i=1}^N \delta_{(m_t^{(i)}, \Sigma_t^{(i)})}$:
\begin{align}
\dot{m}_t^{(i)}& = - \E\left[ \nabla \log\left(\frac{\rho_{\nu_t}}{\tg} \right)\left(Y_t^{(i)}\right)\right]\label{eq:mean_update}\\
\dot{\Sigma}_t^{(i)}&= -\E\left[ \nabla^2 \log\left(\frac{\rho_{\nu_t}}{\tg} \right)\left(Y_t^{(i)}\right)\right]\Sigma_t^{(i)} -\Sigma_t^{(i)}\E\left[ \nabla^2 \log\left(\frac{\rho_{\nu_t}}{\tg} \right)\left(Y_t^{(i)}\right)\right]\label{eq:cov_update}
\end{align}
where $Y_t^{(i)}\sim \mathcal{N}(m_t^{(i)}, \Sigma_t^{(i)})$.\\

\noindent In contrast, in this work we restrict ourselves to Gaussian mixtures $\rho$ that write $\rho_{\mu} = \int_{\R^d} k_{\epsilon}(\theta-y) d\mu(y)$ where $\mu$ is a measure over $\R^d$.  Then the WGF of $\mu \mapsto \KL(\rho_{\mu}|\pi)$, i.e. the GF of this functional over $\cP_2(\R^d)$ is equivalent to the update above from \cite{lambert2022variational}. Indeed if we fix $\nu= \mu \otimes \delta_{\epsilon \Id}$, a Gaussian mixture writes $\rho_{\nu}(\theta)=\rho_{\mu}(\theta)=\int_{\R^d} k_{\epsilon}(\theta-y)d\mu(y)$. In this case we consider the particle system $\mu_t = \frac{1}{N}\sum_{i=1}^N \delta_{m_t^{(i)}}$, we have $\rho_{\mu_t}(\theta)= \frac{1}{N}\sum_{i=1}^N k_{\epsilon}(\theta-m_t^{(i)} )$. The update \eqref{eq:mean_update} becomes:
\begin{align*}
  \dot{m}_t^{(j)} &= - \E\left[ \nabla \log\left(\frac{\rho_{\nu_t}}{\pi} \right)\left(Y_t^{(j)}\right)\right]\\
  &= - \E[\nabla V(Y_t^{(j)})] - \E\left[ \nabla \log\left(\rho_{\mu_t} \right)\left(Y_t^{(j)}\right)\right]\\
  &= - \E[\nabla V(Y_t^{(j)})] - \E\left[ \frac{\sum_{i=1}^N \nabla k_{\epsilon}(Y_t^{(j)}-m_t^{(i)})}{\sum_{i=1}^N k_{\epsilon}(Y_t^{(j)}-m_t^{(i)})}\right]\\
  &= - \int \nabla V(y)\ke(y-m_t^{(j)})dy - \int  \frac{\sum_{i=1}^N \nabla \ke(y-m_t^{(i)})}{\sum_{i=1}^N  \ke(y-m_t^{(i)})}\ke(y-m_t^{(j)})dy
\end{align*}
since $Y_t^{(j)}\sim \mathcal{N}(m_t^{(j)}, \epsilon \Id)$ has density $\ke(\cdot-m_{t}^{(j)})$ hence we obtain the same update as \eqref{eq:particle_ode}.

\section{Lemmas for the proof of \Cref{th:kl_quantization}}\label{sec:technical_lemmas_quantization}

\begin{lemma}\label{lem:decreasing}
For any $x \in \R^{+}$, the function defined on $[0,+\infty[$ by
\begin{equation*}
    B(x) = \frac{x \log x - x + 1}{(x - 1)^2} \text{ if } x>0,
\end{equation*}
and $B(0)=1$ is monotone decreasing in r.
\end{lemma}

\begin{proof}
Firstly, the derivative of $f$ is given by
\begin{equation*}
B'(x) = \frac{2 - \frac{x+1}{x-1} \log(x)}{(x-1)^2}.
\end{equation*}
Recall some inequalities of the log function derived in \cite{topsoe12007some}:
\begin{align*}
&\forall x \in [1, + \infty[ \;, \quad \frac{2 (x-1)}{x+1} \leq \log(x),   \\
&\forall x \in [0, 1] \;, \quad \log(x) \leq \frac{2 (x-1)}{x+1}.
\end{align*}
Consequently, combining those two inequalities and multiplying by $1 / (x-1)$ which is positive on $[1, + \infty[$ and negative on $[0, 1[$ we obtain for any $x \in [0, +\infty[$
\begin{equation*}
\frac{\log(x)}{x - 1} \geq 2 (x-1)
\end{equation*}
It implies that the derivative $f'(x)$ is always negative and $f$ is monotone decreasing.
\end{proof}

\begin{lemma} \label{lem:bound}
For any $x \in \R^{+}$ we have
\begin{equation*}
    C(x) = B(x) (x - 1) = \frac{ x \log(x) - x + 1}{x - 1} \leq \sqrt{x} - 1
\end{equation*}
\end{lemma}
\begin{proof}
Recall the inequalities derived in \cite{topsoe12007some} 
\begin{enumerate}
\item for any $x \in [1, + \infty[$, $\log(x) \leq \frac{x-1}{\sqrt{x}}$
\item for any $x \in [0, 1]$, $\log(x) \geq \frac{x-1}{\sqrt{x}}$\;.
\end{enumerate}
Combining those inequalities and multiplying by $1/(x - 1)$ which is positive on $[1, + \infty[$ and negative on $[0, 1[$, we obtain for any $x \in [0, \infty[$,
\begin{equation*}
\frac{\log(x)}{x-1} \leq \frac{1}{\sqrt{x}}.
\end{equation*}
Moreover, 
\begin{align*}
C(x) - \sqrt{x} - 1 = \frac{x \log(x)}{x-1} - \sqrt{x}.
\end{align*}
Consequently, by multiplying the previous inequality by $x$, we obtain that $C(x) - (\sqrt{x} + 1) \leq 0$
\end{proof}

\begin{lemma}\label{lem:inequality_alpha}
For any $\alpha \in [0, 1]$, we have
\begin{equation*}
    -\alpha + (1-\alpha) \log(1-\alpha) + 2 \alpha \sqrt{1-\alpha} \leq 0.
\end{equation*}
\end{lemma}
\begin{proof}
Let's start by applying the classical inequality $\forall x>-1$, $\log(1 + x) \leq x$ at $x = -\alpha$, we obtain $
    \log(1 - \alpha) \leq - \alpha$. 
Hence,
\begin{align*}
- \alpha + (1-\alpha) \log(1-\alpha) + 2 \alpha \sqrt{1-\alpha} &\leq - \alpha - \alpha (1 - \alpha) + 2 \alpha \sqrt{1-\alpha}\\
&= \alpha ( 2 \sqrt{1-\alpha} - 2 + \alpha)\\
&:= \alpha g(\alpha)
\end{align*}
Moreover, 
\begin{equation*}
g(\alpha) =  2 \sqrt{1-\alpha} - 2 + \alpha \text{ and } g'(\alpha) = \frac{-1}{\sqrt{1-\alpha}} - 1 \leq 0,
\end{equation*}
hence $g$ is decreasing. Consequently, 
\begin{equation*}
- \alpha + (1-\alpha) \log(1-\alpha) + 2 \alpha \sqrt{1-\alpha} \leq \alpha g(0) \leq 0 \qedhere.
\end{equation*}
\end{proof}

\section{Proof of \Cref{prop:decreasing_functional}}\label{sec:proof_prop_decreasing}

We first deal with the potential energy term. 
Notice that under \Cref{ass:smooth_potential}, $\tVe$ is also $L$-smooth, since for any $x,y\in \R^d$
\begin{equation}
\left\Vert\nabla \tVe(x)-\nabla \tVe(y) \right\Vert  \le \int  \ke(\theta)\| \nabla V(x-\theta) - \nabla V(y-\theta)\| d\theta \le L \|x-y\|  
\end{equation}
since $\int \ke(\theta)d\theta = 1.$ Hence we have the following.

\begin{proposition} \label{prop:descent_potential}
Let $\rho$ in $\cP_2(\R^d)$ and $\rho_t = T_{t\#}\rho$ where $T_t = \Id + t\phi$.
    \begin{equation*}
        \frac{d}{dt} \cGe(\rho_t) - \frac{d}{dt} \cGe(\rho_t) \Big\vert_{t=0} \leq  L t \Vert \phi \Vert^2_{L^2(\rho)}. 
    \end{equation*}
\end{proposition}
\begin{proof}
    By the chain rule in Wasserstein space we have
    \begin{equation*}
\frac{d}{dt} \cGe(\rho_t) = \langle \nabla \cGe'(\rho_t), v_t \rangle_{L^2(\rho_t)} =\langle \nabla \tVe, v_t \rangle_{L^2(\rho_t)}.
\end{equation*}
Hence, using the transfer Lemma and Cauchy-Schwarz successively,
\begin{align}
\frac{d}{dt} \cGe(\rho_t) - \frac{d}{dt} \cGe(\rho_t) \Big\vert_{t=0} =& \langle \nabla \tVe, v_t \rangle_{L^2(\rho_t)} -   \langle \nabla \tVe, \phi \rangle_{L^2(\rho)}\nonumber= \langle \nabla \tVe\circ T_t - \nabla \tVe, \phi \rangle_{L^2(\rho)} \nonumber\\
&\leq \E_{w \sim \rho} [L \Vert \|T_t(x) -x\|  \Vert \Vert \phi(w) \Vert] \le L t \|\phi\|^2_{L^2(\rho)}\;.\qedhere
\end{align}
\end{proof}

We now turn to the mollified entropy term that is the most challenging.

\begin{proposition}\label{prop:lipschitz_mollified_entropy}
Let $\rho$ denote a mixture of $n$ Diracs and $\rho_t = T_{t\#}\rho$ where $T_t = \Id + t\phi$. We have:
\begin{equation*}
    \frac{d}{dt} \cUe(\rho_t) - \frac{d}{dt} \cUe(\rho_t) \Big\vert_{t=0} \leq \left(\frac{1}{\epsilon^2} + \frac{\sqrt{m_2(\rho) n}}{ \epsilon^3} + \frac{\sqrt{n}}{\epsilon^2} + \frac{n\sqrt{m_2(\rho)}}{2 \epsilon^3}\right) \; t \Vert \phi \Vert^2_{L^2(\rho)} 
\end{equation*}
where $m_2(\rho)$ denotes the second moment of $\rho$. 
\end{proposition}

\begin{proof}
By the chain rule in Wasserstein space, we have
\begin{equation*}
\frac{d}{dt} \cUe(\rho_t) = \langle \nabla \cUe'(\rho_t), v_t \rangle_{L^2(\rho_t)}.
\end{equation*}
Consequently, 
\begin{align}
\frac{d}{dt} \cUe(\rho_t) - \frac{d}{dt} \cUe(\rho_t) \Big\vert_{t=0} =& \langle \nabla \cUe'(\rho_t), v_t \rangle_{L^2(\rho_t)} -   \langle \nabla \cUe'(\rho), \phi \rangle_{L^2(\rho)}\nonumber\\
&= \langle \nabla \cUe'(T_t \# \rho) \circ T_t - \nabla \cUe'(\rho), \phi \rangle_{L^2(\rho)} \nonumber\\
&\leq \E_{w \sim \rho} [\Vert \nabla \cUe'(T_t \# \rho)(T_t(w)) - \nabla \cUe'(\rho)(w) \Vert \Vert \phi(w) \Vert]\label{ine:first}\;
\end{align}
where in the second line we have used the transfer Lemma and in the last inequality Cauchy-Schwarz. Now, let's focus on the term $\Vert \nabla \cUe'(T_t \# \rho)(T_t(w)) - \nabla \cUe'(\rho)(w) \Vert$, that we will decompose as
\begin{align}
\nabla \cUe'(T_t \# \rho) \circ T_t - \nabla \cUe'(\rho) &=  \nabla \cUe'(T_t \# \rho)(T_t(w)) -  \nabla \cUe'(\rho)(T_t(w)) + \nabla \cUe'(\rho)(T_t(w)) - \nabla \cUe'(\rho)(w) \nonumber \\
&:= \mathcal{B}_{T_t(w)}(\rho_t, \rho) + \mathcal{A}_{\rho}(T_t(w), w) .\label{def:A}
\end{align}
In the rest of the proof, we will show the Lipschitzness on $w$ for $\mathcal{A}$ and in $\rho$ for $\mathcal{B}$.

Using Proposition \ref{prop:A} 
and Proposition \ref{prop:B}, we have
\begin{multline*}
\frac{d}{dt} \cUe(\rho_t) - \frac{d}{dt} \cUe(\rho_t) \Big\vert_{t=0} \leq 
\E_{w\sim \rho } \left[ (\| \mathcal{A}_{\rho}(T_t(w), w)\|+ \| \mathcal{B}_{T_t(w)}(\rho_t, \rho) \| )\| \phi(w)\|\right]\\
\le 
\left(\frac{1}{\epsilon^2} + \frac{\sqrt{m_2(\rho) n}}{2 \epsilon^3}\right) \; t \|\phi \Vert^2_{L^2(\rho)} + \left(\frac{\sqrt{n}}{\epsilon^2} + \frac{\sqrt{n m_2(\rho)}}{2 \epsilon^3} + \frac{n \sqrt{m_2(\rho)}}{2 \epsilon^3}\right) \; t \Vert \phi \Vert^2_{L^2(\rho)}\\
\leq \left(\frac{1}{\epsilon^2} + \frac{2\sqrt{ m_2(\rho) n}}{\epsilon^3} + \frac{\sqrt{n}}{\epsilon^2} + \frac{n \sqrt{m_2(\rho)}}{2 \epsilon^3}\right) \; t \Vert \phi \Vert^2_{L^2(\rho)}.\qedhere
\end{multline*}
\end{proof}

\begin{proposition}\label{prop:A}
Let $\rho$ denote a mixture of $n$ Diracs and $\rho_t = T_{t\#}\rho$ where $T_t = \Id + t\phi$. It holds that
\begin{equation}
\Vert \mathcal{A}_{\rho}(T_t(w), w) \Vert \leq \left(\frac{1}{\epsilon^2} + \frac{\sqrt{2m_2(\rho) n}}{ \epsilon^3}\right) \; t \Vert \phi(w) \Vert     \end{equation}
\end{proposition}
\begin{proof}
Recalling the definition of $\nabla \cUe'$ in \eqref{eq:wgd_mollified_u}, we obtain \begin{equation}\nabla \cUe'(\rho)(w)= \int \ke(\theta - w) \frac{\int \nabla \ke(\theta - y) d\rho(y)}{\int \ke(\theta - y) d\rho(y)} d\theta \nonumber = \frac{1}{\epsilon^2} \int \ke(\theta - w) \frac{\int y \ke(\theta - y) d\rho(y)}{\int \ke(\theta - y) d\rho(y)} d\theta - \frac{w}{\epsilon^2}.\label{def:nablau}\end{equation}
Then from the definition of $\mathcal{A}$ in \eqref{def:A} we have
\begin{align*}
\Vert \mathcal{A}_{\rho}(T_t(w), w) \Vert
&= \frac{1}{\epsilon^2}\left\Vert \int(\ke(\theta - T_t(w)) + \ke(\theta - w)) \frac{\int y \ke(\theta - y) d\rho(y)}{ \ke \star \rho(\theta)} d\theta - T_t(w) + w \right\Vert \\
&\leq \frac{1}{\epsilon^2} \int \vert \ke(\theta - T_t(w)) - \ke(\theta - w) \vert \frac{\int \Vert y \Vert \ke(\theta - y) d\rho(y)}{ \ke \star \rho(\theta)} d\theta + \frac{t\Vert \phi(w) \Vert}{\epsilon^2}\end{align*}
Moreover, recall that $\rho$ is a mixture of $n$ Diracs. Therefore, we have
\begin{multline} 
\int \Vert y \Vert \ke(\theta - y) d\rho(y) \leq \sqrt{\int \Vert y \Vert^2 d\rho(y)} \sqrt{\int \ke(\theta - y)^2 d\rho(y)}
= \sqrt{m_2(\rho)} \sqrt{\frac{1}{n} \sum_{i=1}^{n} \ke(\theta - y_i)^2}  \\
\leq \sqrt{m_2(\rho)} \frac{1}{\sqrt{n}} \sum_{i=1}^{n} \ke(\theta - y_i)
\leq \sqrt{m_2(\rho) n} \; \ke \star \rho(\theta).\label{ref:finite_mixture}
\end{multline}
Consequently,
\begin{align*}
\Vert \mathcal{A}_{\rho}(T_t(w), w) \Vert &\leq  \frac{\sqrt{m_2(\rho) n}}{\epsilon^2} \; 2\TV(\mathcal{N}(T_t(w), \epsilon^2 \Id), \; \mathcal{N}(w, \epsilon^2 \Id) + \frac{t \Vert \phi(w) \Vert}{\epsilon^2} \\
&\leq \frac{\sqrt{m_2(\rho) n}}{\epsilon^2}\; \sqrt{2 \KL(\mathcal{N}(T_t(w), \epsilon^2 \Id), \; \mathcal{N}(w, \epsilon^2 \Id)} +\frac{t \Vert \phi(w) \Vert}{\epsilon^2}\\
&= \left(\frac{1}{\epsilon^2} + \frac{\sqrt{2 m_2(\rho) n}}{\epsilon^3}\right) \; t \Vert \phi(w) \Vert. \qedhere
\end{align*}
\end{proof}
\begin{proposition}\label{prop:B} 
Let $\rho$ denote a mixture of $n$ Diracs and $\rho_t = T_{t\#}\rho$ where $T_t = \Id + t\phi$. We have:
\begin{equation*}
\E_{w \sim \rho}[\Vert \mathcal{B}_{T_t(w)}(\rho_t, \rho) \Vert \Vert \phi(w) \Vert] \leq  \left(\frac{\sqrt{n}}{\epsilon^2} + \frac{\sqrt{n m_2(\rho)}}{ \epsilon^3} + \frac{n \sqrt{m_2(\rho)}}{\epsilon^3}\right) \; t \Vert \phi \Vert^2_{L^2(\rho)}.
\end{equation*}
\end{proposition}
\begin{proof}
Recalling the definition of \Cref{eq:wgd_mollified_u}, we have
\begin{multline}
\E_{w \sim \rho}[\Vert \mathcal{B}_{T_t(w)}(\rho_t, \rho) \Vert \Vert \phi(w) \Vert] 
= \int \left\Vert \frac{1}{\epsilon^2} \int \ke(\theta - T_t(w)) \frac{\int y \ke(\theta - y) d\rho_t(y)}{\ke \star \rho_t(\theta)} - \frac{\int y\ke(\theta - y) d\rho(y)}{\ke \star \rho(\theta)} \right\Vert. \Vert \phi(w) \Vert d\theta d\rho(w)\\
\leq \frac{1}{\epsilon^2} \int  \Big(\int \ke(\theta - T_t(w)) \Vert \phi(w) \Vert d\rho(w) \Big) \; \Big\Vert \frac{\int y \ke(\theta - y) d\rho_t(y)}{\ke \star \rho_t(\theta)} - \frac{\int y\ke(\theta - y) d\rho(y)}{\ke \star \rho(\theta)} \Big \Vert  d\theta.  \label{ine:B}
\end{multline}
Then, we can use Cauchy Schwarz inequality and that $\rho_t$ is supported on $n$ Diracs to obtain
\begin{equation}
\int \Vert \phi(w) \Vert \ke(\theta - T_t(w)) d\rho(w) 
\leq \Vert \phi \Vert_{L^2(\rho)} \sqrt{\int \ke(\theta - w)^2 d\rho_t(w)}
=  \sqrt{n} \Vert \phi \Vert_{L^2(\rho)} \; \ke \star \rho_t(\theta). \label{ine:CS}
\end{equation}
Moreover, recall that
\begin{multline}
\Big \Vert \frac{\int y \ke(\theta - y) d\rho_t(y)}{\ke \star \rho_t(\theta)} - \frac{\int y \ke(\theta - y) d\rho(y)}{\ke \star \rho(\theta)} \Big \Vert\\
\leq \frac{\Big \Vert \int y \ke(\theta - y) d\rho_t(y) - \int y  \ke(\theta - y) d\rho(y)\Big \Vert}{\ke \star \rho_t(\theta)}
+ \Big \Vert \int y \ke(\theta - y) d\rho(y) \Big \Vert \; \Big \vert \frac{1}{\ke \star \rho_t(\theta)} - \frac{1}{\ke \star \rho(\theta)} \Big \vert \\
\leq \frac{\int \Vert T_t(y) \ke(\theta - T_t(y))  - y  \ke(\theta - y) \Vert d\rho(y)}{\ke \star \rho_t(\theta)} 
+  \int \Vert y \Vert \ke(\theta - y) d\rho(y) \; \Big \vert \frac{\ke \star \rho(\theta) - \ke \star \rho_t(\theta)}{\ke \star \rho(\theta) \; \ke \star \rho_t(\theta)} \Big \vert \\
:= \mathcal{C}_1(\theta) + \mathcal{C}_2(\theta). \label{ine:gl}
\end{multline}
We can now combine inequalities \ref{ine:B}, \ref{ine:CS} and \ref{ine:gl} to obtain
\begin{align*}
\E_{w \sim \rho}[\Vert \mathcal{B}_{T_t(w)}(\rho_t, \rho) \Vert \Vert \phi(w) \Vert] \leq \frac{\sqrt{n} \Vert \phi \Vert_{L^2(\rho)}}{\epsilon^2} \int \ke \star \rho_t(\theta) (\mathcal{C}_1(\theta) + \mathcal{C}_2(\theta)) d\theta.
\end{align*}
We first focus on the $\mathcal{C}_1$ term:
\begin{align*}
\int \ke \star \rho_t(\theta) \mathcal{C}_1(\theta) d\theta &= \int \int \Vert T_t(y) \ke(\theta - T_t(y)) - y \ke(\theta - y) \Vert d\rho(y) d\theta \\
&\leq \int \int \Vert T_t(y) - y \Vert \ke(\theta - T_t(y)) +  \Vert y \Vert \; \vert \ke(\theta - T_t(y)) - \ke(\theta - y) \vert d\rho(y) d\theta\\
&\leq t \E_{y \sim \rho}[\Vert \phi(y) \Vert] + \int \Vert y \Vert 2\TV(\mathcal{N}(T_t(y), \epsilon^2 \Id), \mathcal{N}(y, \epsilon^2 \Id)) d\rho(y)\\
&\leq t \Vert \phi \Vert_{L^2(\rho)} + \int \frac{ t \Vert y \Vert \; \Vert \phi(y) \Vert}{ \epsilon} d\rho(y)\\
&\leq t \Vert \phi \Vert_{L^2(\rho)} + \frac{t \Vert \phi \Vert_{L^2(\rho)}}{2 \epsilon} \sqrt{\int \Vert y \Vert^2 d\rho(y)} \\
&= \left(1 + \frac{\sqrt{m_2(\rho)}}{ \epsilon}\right) \; t\Vert \phi \Vert_{L^2(\rho)}.
\end{align*}
Finally, we focus on the $\mathcal{C}_2$ term. We obtain using the same computations as in \eqref{ref:finite_mixture}:
\begin{align*}
\int \ke \star \rho_t(\theta) \mathcal{C}_2(\theta) d\theta &= \int \frac{\int \Vert y \Vert \ke(\theta - y) d\rho(y)}{\ke \star \rho(\theta)} \; \vert \ke \star \rho(\theta) - \ke \star \rho_t(\theta) \vert d\theta\\
&\leq \sqrt{m_2(\rho) n} \int \vert \ke \star \rho(\theta) - \ke \star \rho_t(\theta) \vert d\theta \\
&\leq \sqrt{m_2(\rho) n} \int \int \vert \ke(\theta - y) - \ke(\theta - T_t(y)) \vert d\rho(y) d\theta\\
&\leq \sqrt{m_2(\rho) n} \int \frac{t \Vert \phi \Vert_{L^2(\rho)}}{ \epsilon} d\rho(y) d\theta\\
&\leq \frac{t \sqrt{m_2(\rho) n}}{\epsilon} \Vert \phi \Vert_{L^2(\rho)}.
\end{align*}
Combining the previous inequalities, we obtain
\begin{equation*}
\E_{w \sim \rho}[\Vert \mathcal{B}_{T_t(w)}(\rho_t, \rho) \Vert \Vert \phi(w) \Vert] \leq \left(\frac{\sqrt{n}}{\epsilon^2} + \frac{\sqrt{n m_2(\rho)}}{ \epsilon^3} + \frac{n\sqrt{m_2(\rho)}}{ \epsilon^3}\right) \; t \Vert \phi \Vert^2_{L^2(\rho)}. \qedhere
\end{equation*}
\end{proof}

\section{Wasserstein Hessians of relative entropies}\label{sec:wasserstein_hessians}

\subsection{Proof of \Cref{prop:hessian_kl}}\label{sec:proof_hessian_KL}

\begin{proof}

Let $\mu_t = (\Id + t \nabla \psi)_{\#}\mu$ where $\psi \in C_c^{\infty}(\R^d)$. Let $\mu_t, \mu^*$ be the densities of $\mu_t$ and $\mu^*$ respectively. We denote by $\phi_t = \Id + t g$ where $g = \nabla \psi$. Hence we have $\J \phi_t = \Id + t \J g$. Time derivatives are denoted as $\phi'_t = \frac{d \phi_t}{dt}$. Notice that $(\J \phi_t)' = \J \phi'_t = \J g = \Hess \psi$.

For any $f$-divergence, 
\begin{equation*}
    h_{\mu}(t) = \int f\left(\frac{\mu_t(x)}{\mu^*(x)}\right) \mu^*(x) \dd x = \int \tilde{f}\left(\frac{\mu_t(x)}{\mu^*(x)}\right) \mu_t(x) \dd x.
\end{equation*}
where $\tilde{f}(t) = f(t)/t$. By the transfer lemma and change of variables formula, we have
\begin{equation*}
    h_{\mu}(t) = \int \tilde{f}\left(\frac{\mu(x)}{\mu^*(\phi_t(x))|\J \phi_t(x)|}\right) \dd \mu(x).
\end{equation*}
Let us rewrite
\begin{equation*}
    h_{\mu}(t) = \int \tilde{f}\left(\mu(x) e^{n_t(x)}\right) \dd \mu(x), \quad \text{where} \quad n_t(x) = V(\phi_t(x)) - \log|\J \phi_t(x)|.
\end{equation*}
We have consecutively:
\begin{align*}
    h'_{\mu}(t) &= \int \tilde{f}'\left(\mu(x) e^{n_t(x)}\right) \mu(x) n'_t(x) e^{n_t(x)} \dd \mu(x) \\
    h''_{\mu}(t) &= \int \left[\tilde{f}''(\mu(x) e^{n_t(x)}) \left(\mu(x) n'_t(x) e^{n_t(x)}\right)^2\right. \\
    &\left. \tilde{f}'(\mu(x) e^{n_t(x)}) \left(n''_t(x) + n'_t(x)^2\right) \mu(x) e^{n_t(x)}\right] \dd x
\end{align*}
where
\begin{align*}
    n'_t(x) &= \ps{\nabla V(\phi_t(x)), \phi'_t(x)} - \Tr\left((\J \phi_t(x))^{-1} \J \phi'_t(x)\right), \\
    n''_t(x) &= \ps{\hess_V (\phi_t(x)) \phi'_t(x), \phi'_t(x)} + \Tr\left((\J \phi_t(x))^{-1} \J \phi'_t(x))^2\right),
\end{align*}
since $\phi''_t = 0$. At time $t=0$, we have
\begin{align*}
    n_0(x) &= V(x) = -\log(\mu^*(x)) \\
    n'_0(x) &= \ps{\nabla V(x), \nabla \psi(x)} - \Delta \psi(x), \\
    n''_0(x) &= \ps{\hess_V(x)\nabla \psi(x), \nabla \psi(x)} + \|\hess \psi(x)\|^2_{HS}
\end{align*}
since $\Tr(\hess \psi) = \Delta \Psi$ and $\Tr((\hess \psi)^2) = \|\hess \psi\|^2_{HS}$. Notice that $n'_0(x) = \mathcal{L}_{\mu^*}\psi(x)$ where $\mathcal{L}_{\mu^*}:\psi \mapsto \ps{\nabla V, \nabla \psi} - \Delta \psi$ denotes the (negative) generator of the standard Langevin diffusion with stationary distribution $\mu^*$ with density $\mu^*\propto e^{-V}$, see \citet[Section 4.5]{pavliotis2014stochastic}.

Now we get at time $t=0$:
\begin{multline*}
    h''_{\mu}(0) = \int \left[\left(\tilde{f}''\left(\frac{\mu(x)}{\mu^*(x)}\right) \left(\frac{\mu(x)}{\mu^*(x)}\right)^2 + \tilde{f}'\left(\frac{\mu(x)}{\mu^*(x)}\right) \left(\frac{\mu(x)}{\mu^*(x)}\right)\right) \left(\mathcal{L}_{\mu^*}\psi(x)\right)^2 \right.\ \\
    \left. + \tilde{f}'\left(\frac{\mu(x)}{\mu^*(x)}\right) \left(\frac{\mu(x)}{\mu^*(x)}\right) \left(\ps{\hess_V(x)\nabla \psi(x), \nabla \psi(x)} + \|\hess \psi(x)\|^2_{HS}\right)\right] \mu(x) \dd x.
\end{multline*}
Hence if $V$ is convex, and that $\min(\tilde{f}'(t), t\tilde{f}'(t) + t^2\tilde{f}''(t)) \ge 0$, then $h''_{\mu}(0) \ge 0$. Now let $f(t) = t\log t - t$, then $h_{\mu}(t) = \KL(\mu_t|\mu^*)-1$. Then, $\tilde{f}(t) = \log(t)-1$; $\tilde{f}'(t) = 1/t, \tilde{f}''(t) = -1/t^2$, hence $t\tilde{f}'(t) + t^2\tilde{f}''(t) = 0$ and we obtain more precisely:
\begin{equation*}
    \Hess_{\mu}\KL(\psi, \psi) = \int \left[\ps{\hess_V(x)\nabla \psi(x), \nabla \psi(x)} + \|\hess \psi(x)\|^2_{HS}\right] \mu(x) \dd x.\qedhere
\end{equation*}

\end{proof}

\subsection{Hessian of the mollified relative entropy}\label{sec:hessian_mollified}
Recall that $\cFe(\mu)=\cGe(\mu) +\cUe(\mu)$.  Hence, for any $\psi \in C_c^{\infty}(\R^d)$, $\Hess_{\mu}\cFe(\psi,\psi) =\Hess_{\mu}\cGe(\psi,\psi)+\Hess_{\mu}\cUe(\psi,\psi)$. We directly have for the potential energy part that 
\begin{equation}
    \frac{d^2 \cGe(\rho_t)}{dt^2} \Bigr|_{\substack{t=0}}= \int \ps{\hess_{\tVe}(x)\nabla \psi(x), \nabla \psi(x)} d\mu(x).
\end{equation}
using again our computation from \Cref{sec:proof_hessian_KL}. Since $\hess_{\tVe} = \ke \star \hess_V$ and $\ke$ converges to a Dirac at origin as $\epsilon$ goes to zero, we get $\Hess_{\mu}\cGe(\psi,\psi)\xrightarrow[\varepsilon \to 0]{}\int \ps{\hess_{V}(x)\nabla \psi(x), \nabla \psi(x)} d\mu(x).$

We now turn to the mollified entropy part. 
We rewrite it along a geodesic $(\rho_t,v_t)_{t\in[0,1]}$ as
\begin{equation*}
    \cUe(\rho_t) = \int \log(\ke\star \rho_t) d(\ke\star \rho_t)
    =\int_{\theta} U(\ke\star \rho_t(\theta))d\mathcal{L}_d(\theta),
\end{equation*}
denoting $U:x\mapsto x\log(x)$.  The first time derivative of $t\mapsto  \cUe(\rho_t)$ is:
\begin{align}\label{eq:first_der_reg_entropy}
\frac{d \cUe(\rho_{t})}{dt} &= 
\int U'(\ke\star \rho_t(\theta) )  \dv[]{}{t} \ke\star \rho_t(\theta)\dd\theta\\
&= \int (1+\log(\ke\star \rho_t(\theta))\int \ps{\nabla \ke(\theta-x),v_t(x)} \dd \rho_t(x)\dd\theta
\end{align}
Since by an integration by parts,
\begin{equation*}
    \dv[]{}{t} \ke\star \rho_t(\theta)\dd\theta = \int \ke(\theta-x)\frac{\partial \rho_t(x)}{\partial t}\dd x = \int \nabla \ke(\theta-x)\rho_t(x)v_t(x)\dd x.
\end{equation*}
From \Cref{eq:first_der_reg_entropy} we obtain
\begin{align}
\frac{d^2 \cUe(\mu_{t})}{dt^2} &=\int\left[ U^{''}(\ke\star \rho_t(\theta))\left(\frac{d \ke\star \rho_t(\theta)}{dt} \right)^2 + U^{'}(\ke\star \rho_t(\theta))\frac{d^2 \ke\star \rho_t(\theta)}{dt^2} \right]\dd\theta \nonumber\\
   &\label{eq:hessian_reg_entropy}
   = \int_{\theta}\left[(\ke\star \rho_t(\theta))^{-1} \left(\frac{d \ke\star \rho_t(\theta)}{dt} \right)^2
 +\left(1 + \log( \ke\star \rho_t(\theta)) \right)\frac{d^2 \ke\star \rho_t(\theta)}{dt^2} \right]\dd\theta.
\end{align}
The first term in \eqref{eq:hessian_reg_entropy} is always positive but the second may not because of the logarithmic term. However, as $\epsilon \to 0$, we recover the geodesic convexity of the negative entropy, as stated in the following proposition. 

\begin{proposition}\label{prop:limiting_hessian}
Let $\mu \in \cP_2(\R^d)$. Let  $\psi \in C_c^{\infty}(\R^d)$. As $\epsilon\to 0$, the Wasserstein Hessian of the regularized entropy $\cUe$ converges to the one of the regular negative entropy $\cU(\mu)=\int \log(\mu)\dd \mu$, i.e:
\begin{equation}
 \Hess_{\mu}\cUe(\psi,\psi)
 \xrightarrow[\varepsilon \to 0]{} \Hess_{\mu}\cU(\psi,\psi)=  \int  \| \hess \psi(x)\|^2_{HS}  \dd \mu(x).
\end{equation}
\end{proposition}
\begin{proof}
    
For each term, we will first take the limit as $t\to 0$ to recover the definition of the Hessian at $\mu$ (limiting distribution of $\rho_t$ as $t$ goes to 0), then $\epsilon \to 0$ to recover the case of the standard (non-regularized) relative entropy. Denote $h_{\mu}^{\epsilon}(t,\theta)= \ke\star \rho_t(\theta)=\int \ke(\theta-x)d\rho_t(x) = \int \ke(\theta-\phi_t(x))d\mu(x)$ by the transfer lemma.
We have
\begin{align}
\frac{d^2 \cUe(\mu_{t})}{dt^2} &=\int\left[ U^{''}(h_{\mu}^{\epsilon}(t,\theta))\left(\frac{d h_{\mu}^{\epsilon}(t,\theta)}{dt} \right)^2 + U^{'}(h_{\mu}^{\epsilon}(t,\theta))\frac{d^2 h_{\mu}^{\epsilon}(t,\theta)}{dt^2} \right]\dd\theta \nonumber\\
   &\label{eq:hessian_reg_entropy}
   = \int_{\theta}\left[(h_{\mu}^{\epsilon}(t,\theta))^{-1} \left(\frac{d h_{\mu}^{\epsilon}(t,\theta)}{dt} \right)^2
 +\left(1 + \log( h_{\mu}^{\epsilon}(t,\theta)) \right)\frac{d^2 h_{\mu}^{\epsilon}(t,\theta)}{dt^2} \right]\dd\theta.
\end{align}
 We firstly have
\begin{equation*}
     h_{\mu}^{\epsilon}(t,\theta)= \ke\star \rho_t(\theta)\xrightarrow[t \to 0]{}\ke\star \mu(\theta) \xrightarrow[\varepsilon \to 0]{}\mu(\theta).
\end{equation*}
Recall that the continuity equation along Wasserstein geodesics write:
\begin{equation}\label{eq:cew}
    \frac{\partial \rho_t(x)}{\partial t} + \div(\rho_t(x) \nabla \psi \circ \phi_t^{-1}(x)))=0.
\end{equation}
Then, using $\div(aB)= \ps{\nabla a,B}+ a\div(B)$, the first time derivative of $t \mapsto h_{\mu}^{\epsilon}(t,\theta)$ writes
\begin{align}
    \label{eq:hmutdot}
    \frac{d h_{\mu}^{\varepsilon}(t,\theta)}{dt}
        &=   \int \ke(\theta-x)\frac{\partial \rho_t(x)}{\partial t}\dd x\\
    &=  - \int \ke(\theta-x)\div(\rho_t(x)\nabla\psi(\phi_t^{-1}(x)))\dd x\\
    &\xrightarrow[t \to 0]{}  - \int \ke(\theta-x)\div(\mu(x)\nabla\psi(x))\dd x
    ,
\end{align}
Hence for the first term in \eqref{eq:hessian_reg_entropy} we have:
\begin{align}
    \int ( h_{\mu}^{\epsilon}(t,\theta) )^{-1}& \left(\frac{d h_{\mu}^{\varepsilon}(t,\theta)}{dt}\right)^2 \dd\theta 
    \\
    &\xrightarrow[t \to 0]{} \int \left(\ke \star \mu(\theta) \right)^{-1} \left( \int \ke(\theta-x)\div(\mu(x)\nabla\psi(x))\dd x\right)^2\dd \theta\\
    &\xrightarrow[\varepsilon \to 0]{} \int \mu(\theta)^{-1}\div(\mu(\theta)\nabla\psi(\theta))^2 \dd \theta\\
    &= 
    \int \mu(\theta)^{-1} \langle \nabla \mu(\theta), \nabla \psi(\theta)\rangle^2 d\theta
    + \int \Delta \psi(\theta)^2d\mu(\theta)
    + 2 \int \Delta \psi(\theta) \ps{\nabla \mu(\theta), \nabla \psi(\theta)}\dd \theta\\
    &=(a)+(b)+(c) \label{eq:first_term_hessian_reg_entropy}.
\end{align}
We now turn to second term in \eqref{eq:hessian_reg_entropy}. Using \eqref{eq:hmutdot}, the second time derivative of $h_{\mu}^{\epsilon}$ writes:
\begin{align*}
  \frac{d^2 h_{\mu}^{\varepsilon}(t,\theta)}{dt^2}
  &=-\int \ke(\theta-x)\div(\frac{d}{dt}(\rho_t(x)\nabla\psi(\phi_t^{-1}(x))))\dd x\\
  &=- \int \ke(\theta-x) \div( \frac{\partial \rho_t(x)}{\partial t} \nabla \psi(\phi_t^{-1}(x))) dx - \int \ke(\theta-x) \div( \rho_t(x)  \frac{d \nabla \psi(\phi_t^{-1}(x))}{dt} ) \dd x\\
  &=(d)+(e).
\end{align*}
Then using $\frac{\partial \rho_t(x)}{\partial t}  = - \div (\rho_t(x)\nabla \psi(\phi_t^{-1}(x)))= -\langle \nabla \rho_t(x), \nabla \psi(\phi_t^{-1}(x))\rangle - \rho_t(x) \Delta \psi(\phi_t^{-1}(x))$, we have 
\begin{align*}
(d) 
    &= \int \ke(\theta-x) \div(\langle \nabla \rho_t(x), \nabla\psi(\phi_t^{-1}(x)) \rangle + \rho_t(x) \Delta \psi(\phi_t^{-1}(x))) \nabla(\psi(\phi_t^{-1}(x))) \dd x\\
    &\xrightarrow[t \to 0]{}\int \ke(\theta-x) \div(\langle \nabla \mu(x), \nabla\psi(x) \rangle + \mu(x) \Delta \psi(x)) \nabla(\psi(x)) \dd x\\
    &\xrightarrow[\epsilon \to 0]{} \div\Big(\langle \nabla \mu(\theta), \nabla \psi(\theta) \rangle \nabla\psi(\theta)\Big) + \div\Big(\mu(\theta) \Delta \psi(\theta)\nabla(\psi(\theta))\Big)
\end{align*}

Now, using $\phi_t^{-1}\approx \Id - t\nabla \psi$ for $t\approx 0$:
\begin{multline*}
 (e) = - \int \ke(\theta-x) \div(\rho_t(x) \frac{d}{dt}(\nabla \psi(\phi_t^{-1}(x)))) dx\\
    \xrightarrow[t \to 0]{} \int \ke(\theta-x) \div(\mu(x) \hess\psi(x) \nabla \psi(x)) \dd x
    \xrightarrow[\epsilon \to 0]{} \div(\mu(\theta) \hess\psi(\theta) \nabla \psi(\theta)) \dd x.
\end{multline*}

Finally, for the second term in \eqref{eq:hessian_reg_entropy} we have
\begin{align*}
 \int (1 + \log&(\h(t, \theta))) \frac{d^2 \h(t, \theta))}{dt^2} d\theta \\
&\xrightarrow[t, \epsilon \to 0]{} \int (1 + \log(\mu(\theta))) \Big\{ \div(\langle \nabla \mu(\theta), \nabla \psi(\theta) \rangle \nabla \psi(\theta) + \mu(\theta) \Delta \psi(\theta) \nabla \psi(\theta) + \mu(\theta) \hess\psi(\theta) \nabla \psi(\theta)) \Big\} \dd\theta\\
    &= -\int \langle \nabla \log(\mu(\theta)), \langle \nabla \mu(\theta), \nabla \psi(\theta) \rangle \nabla \psi(\theta) + \mu(\theta) \Delta \psi(\theta) \nabla \psi(\theta) + \mu(\theta) \hess\psi(\theta) \nabla \psi(\theta)\rangle \dd\theta\\
    &= -\int \mu(\theta)^{-1} \langle \nabla \mu(\theta), \nabla \psi(\theta)\rangle^2 \dd\theta - \int  \Delta \psi(\theta) \langle \nabla \mu(\theta), \nabla \psi(\theta) \rangle \dd\theta -\int \langle \nabla \mu(\theta), \hess\psi(\theta) \nabla \psi(\theta)\rangle \dd\theta\\
    &= -(a) - \frac{1}{2}(c) -\int \langle \nabla \mu(\theta), \hess\psi(\theta) \nabla \psi(\theta)\rangle \dd\theta.
\end{align*}
Moreover, by an integration by parts, using the divergence of matrix vector product $\div(Ab)=\div(A)b+\Tr(A \nabla b)$:
\begin{align*}
   - \int \langle \nabla \mu(\theta), \hess\psi(\theta) \nabla \psi(\theta)\rangle \dd\theta
    &= \int \div(\hess\psi(\theta) \nabla \psi(\theta)) \dd\mu(\theta)\\
    &=  \int \langle \div \hess\psi(\theta), \nabla \psi(\theta)\rangle + \Tr(\hess\psi(\theta)\hess\psi(\theta))^{\top} \dd\mu(\theta)\\
    &= \int \langle \nabla (\Delta \psi(\theta)), \nabla \psi(\theta)\rangle + \Vert \hess\psi(\theta)\Vert_{F}^2 \dd\mu(\theta),
\end{align*}
where
\begin{align*}
    \int \langle \nabla (\Delta \psi(\theta)), \nabla \psi(\theta)\rangle d\mu(\theta) &= -\int \Delta\psi(\theta) \div(\mu(\theta) \nabla \psi(\theta)) \dd\theta\\
    &= -  \int \Delta\psi(\theta) (\langle \nabla \mu(\theta), \nabla \psi(\theta)\rangle -\mu(\theta) \Delta \psi(\theta)) \dd\theta\\
   &= -\frac{1}{2}(c)-(b).
\end{align*}
Consequently,
\begin{align}\label{eq:second_term_hessian_reg_entropy}
    \int (1 + \log(\h(t, \theta))) \frac{d^2 \h(t, \theta)}{dt^2} d\theta \xrightarrow[t,\epsilon \to 0]{} -(a) - \frac{1}{2} (c) - \frac{1}{2} (c) - (b) +  \int \Vert \hess\psi(\theta) \Vert_F^2 \dd\mu(\theta).
\end{align}
Finally combining \eqref{eq:first_term_hessian_reg_entropy} and \eqref{eq:second_term_hessian_reg_entropy} we get the result. 
\end{proof}

\section{Experimental setting}\label{sec:numeric}

\paragraph{The target distribution}
The target distribution $\tg$ is chosen to be a Gaussian mixture with $100$ components:
\begin{equation*}
 \tg = \frac{1}{100} \sum_{i=1}^{100} \mathcal{N}(x_i^{\star}, \epsilon^2 \mathrm{I}_d)
\end{equation*}
 The components $(x_i^{\star})_{i \leq 100}$ are randomly sampled from a normal distribution $\mathcal{N}(0, \sigma^2 \mathrm{I}_d)$, where $\sigma = 5$ in all experiments. The standard deviation of the target is set to $\epsilon = \epsilon_0 \sqrt{d}$, where $\epsilon_0 = 1$ in our setting. This standard deviation scales with $\sqrt{d}$ because the term $\Vert x_i^{\star}\Vert_2$ also scales with $\sqrt{d}$. Without this scaling, the term $\mathcal{N}(x_i^{\star}, \epsilon^2 \mathrm{I}_d)$ would be very close to a Dirac mass in high dimensions.

 \paragraph{Variational family}

 The variational family used for the experiments is the family of Gaussian mixtures with $10$ components:
 \begin{equation*}
     \mathcal{C}_n = \left\{ \frac{1}{n} \sum_{i=1}^n \mathcal{N}(x_i, \epsilon^2 \mathrm{I}_d),\; x_i \in \R^d \right\}
 \end{equation*}
 At the beginning of the training, the mean of each component $(x_i)_{i\leq 10}$ is randomly initialized, sampled from a normal distribution $\mathcal{N}(0, \zeta^2 \mathrm{I}_d)$, where $\zeta = 15$ in all experiments. Note that these components are initialized further than the parameters $(x_i^{\star})_{i \leq 100}$ of the target $\tg$, ie, $\zeta \geq \sigma$. It seems that this setting allows to slightly improve the performances and the mode coverage of the algorithm. For simplicity, the variational family shares the same standard deviation $\epsilon$ than the target.

 \paragraph{Training parameters}
 The step-size is set as $\gamma = \gamma_0 \cdot d$, where $\gamma_0 = 0.01$. According to \cref{prop:decreasing_functional}, the step-size should satisfy $\gamma \leq 2/M$ to ensure a decrease in the objective of each iteration, where the constant $M$ scales inversely with $d$. Therefore, we opted for $\gamma$ to scale with $d$ accordingly.

 \paragraph{Monte Carlo approximation of the cumulative mean}
Let $\mu$ a Gaussian mixture with $n$ components. We denote by $(x_i)_{i \leq n}$ the mean of those components. Therefore, the term $\|\nabla \cFe'(\mu)\|^2_{L^2(\mu)}$ can be approximated by Monte Carlo with $B$ samples using

\begin{align*}
     \|\nabla \cFe'(\mu)\|^2_{L^2(\mu)} &= \int \|\nabla \cFe'(\mu)(w)\|_2^2  d\mu(w)\\
     &=\frac{1}{n} \sum_{i=1}^n \|\nabla \cFe'(\mu)(x_i)\|_2^2\\
&=\frac{1}{n} \sum_{i=1}^n \Vert \int \nabla \log (\frac{\mu(y)}{\tg(y)}) d\ke^{x_i}(y) \Vert_2^2\\
&\approx \frac{1}{n} \sum_{i=1}^n \Vert \frac{1}{B} \sum_{j=1}^B \nabla \log (\frac{\mu(y_j^{i})}{\tg(y_j^{i})}) \Vert_2^2 \;,
\end{align*}
where $y_j^{i} \sim \mathcal{N}(x_i, \epsilon^2 \mathrm{I}_d)$.

\paragraph{Monte Carlo approximation of the KL}
Let $\nu_n$ a Gaussian mixture with $n$ components. We denote by $(x_i)_{i \leq n}$ the mean of those components. Therefore, the Kullback-Leibler divergence between $\nu_n$ and the target $\tg$ can be approximated by Monte Carlo with $B$ samples using
\begin{align*}
    \KL(\nu_n, \tg) &= \int \log(\frac{\nu_n(y)}{\tg(y)}) d\mu(y) \\
&= \frac{1}{n} \sum_{i=1}^{n} \int \log(\frac{\nu_n(y)}{\tg(y)}) d\ke^{x_i}(y)\\
&\approx \frac{1}{B \cdot n} \sum_{i=1}^{n} \sum_{j=1}^{B} \log(\frac{\nu_n(y_j^{i})}{\tg(y_j^{i})}) \;,
\end{align*}
where $y_j^{i} \sim \mathcal{N}(x_i, \epsilon^2 \mathrm{I}_d)$.


\end{document}